\documentclass{article}

\usepackage{amsmath,amssymb,amsfonts}
\usepackage{enumerate}
\usepackage{amsthm,amssymb,amsfonts}
\usepackage{algpseudocode}
\usepackage{algorithm}
\usepackage{hyperref}
\usepackage{color}
\usepackage{caption}
\usepackage{subcaption}

\usepackage[normalem]{ulem}

\newif\ifcomments
\commentstrue 

\newcommand{\red}[1]{\textcolor{red}{#1}}

\newcommand{\R}{\mathbb{R}}
\newcommand{\N}{\mathbb{N}}

\newcommand{\abs}[1]{ \left\vert #1\right\vert }
\newcommand{\norm}[1]{ \left\Vert #1\right\Vert }
\newcommand{\set}[1]{\left\{#1\right\}}
\newcommand{\eval}[2]{\underset{{#1}}{\mathbb{E}}\left[#2\right]}
\newcommand{\Prob}[2]{\underset{{#1}}{\Pr}\left(#2\right)}
\newcommand{\st}{\;.\;}

\newcommand{\boolhc}{\set{0,1}^n}

\newcommand{\poly}{\text{poly}}
\newcommand{\supp}{\text{supp}}
\newcommand{\sgn}{\text{sgn}}

\newcommand{\ltf}{\mathsf{LTF}}

\newcommand{\ltfreal}{\ltf_{\mathbb{R}^n}}
\newcommand{\given}{\;|\;}

\newcommand{\risk}{\mathsf{R}}

\newcommand{\adv}{\mathbb{A}}
\newcommand{\LEQ}{\mathsf{LEQ}}
\newcommand{\EQ}{\mathsf{EQ}}
\newcommand{\EX}{\mathsf{EX}}
\newcommand{\LMQ}{\mathsf{LMQ}}
\newcommand{\MQ}{\mathsf{MQ}}
\newcommand{\PAC}{\mathsf{PAC}}
\newcommand{\VC}{\mathsf{VC}}

\newcommand{\VCH}{\VC(\mathcal{H})}
\newcommand{\RLossLong}{\mathcal{L}_\rho(\C,\H)}
\newcommand{\RLoss}{\mathcal{L}_\rho(\C)}
\newcommand{\Rloss}{\ell_\rho(c,h)}
\newcommand{\RVClong}{\VC(\RLossLong)}
\newcommand{\RVC}{\VC(\RLoss)}

\newcommand{\DVCH}{\VC^*(\mathcal{H})}
\newcommand{\Lit}{\mathsf{Lit}}
\newcommand{\A}{\mathcal{A}}
\newcommand{\C}{\mathcal{C}}

\newcommand{\D}{\mathcal{D}}

\newcommand{\E}{\mathcal{E}}
\newcommand{\F}{\mathcal{F}}

\renewcommand{\H}{\mathcal{H}}

\newcommand{\U}{\mathcal{U}}
\newcommand{\X}{\mathcal{X}}
\newcommand{\x}{\mathbf{x}}
\newcommand{\vv}{\mathbf{v}}

\newcommand{\MonConj}{\textsf{MON-CONJ}}
\newcommand{\Conj}{\textsf{CONJUNCTIONS}}
\newcommand{\Halfspaces}{\textsf{LTF}}
\newcommand{\thresholds}{\mathsf{THRESHOLDS}}

\newtheorem{theorem}{Theorem}
\newtheorem{proposition}[theorem]{Proposition}
\newtheorem{lemma}[theorem]{Lemma}
\newtheorem{corollary}[theorem]{Corollary}
\newtheorem{example}[theorem]{Example}

\newtheorem{definition}[theorem]{Definition}

\theoremstyle{remark}
\newtheorem{remark}[theorem]{Remark}

  \RequirePackage{natbib}
\usepackage{authblk}
\usepackage[lmargin=1in,rmargin=1in, bmargin=1.25in, tmargin=1.25in]{geometry}

\usepackage[utf8]{inputenc} 
\usepackage[T1]{fontenc}    
\usepackage{hyperref}       
\usepackage{url}            
\usepackage{booktabs}       
\usepackage{amsfonts}       
\usepackage{nicefrac}       
\usepackage{microtype}      
\usepackage{xcolor}         
\usepackage{graphicx}

\title{When are Local Queries Useful for Robust Learning?}

\date{}

\author{Pascale Gourdeau}
\author{Varun Kanade}
\author{Marta Kwiatkowska}
\author{James Worrell}

\affil{University of Oxford}

\begin{document}

\maketitle

\begin{abstract}
  Distributional assumptions have been shown to be necessary for the robust 
  learnability of concept classes when considering the exact-in-the-ball robust risk and access to random examples by Gourdeau et al. (2019).
  In this paper, we study learning models where the learner is given more power 
  through the use of \emph{local} queries, and give the first \emph{distribution-free} algorithms that perform robust empirical risk minimization (ERM) for this notion of robustness. 
  The first learning model we consider uses local membership queries (LMQ), where the 
  learner can query the label of points near the training sample.
  We show that, under the uniform distribution, LMQs do not increase the robustness threshold of conjunctions and any superclass, e.g., decision lists and halfspaces.
  Faced with this negative result, we introduce the local \emph{equivalence} query ($\LEQ$) oracle, which returns whether the hypothesis and target concept agree in the perturbation region around a point in the training sample, as well as a counterexample if it exists.
  We show a separation result: on the one hand, if the query radius $\lambda$ is strictly smaller than the adversary's  perturbation budget $\rho$, then distribution-free robust learning is impossible for a wide variety of concept classes; on the other hand, the setting $\lambda=\rho$ allows us to develop robust ERM algorithms.
  We then bound the query complexity of these algorithms based on online learning guarantees and further improve these bounds for the special case of conjunctions. 
We finish by giving robust learning algorithms for halfspaces on $\boolhc$ and then obtaining robustness guarantees for halfspaces in $\mathbb{R}^n$ against \emph{precision-bounded} adversaries.
\end{abstract}

\section{Introduction}

Adversarial examples have been widely studied since the work of \citep{dalvi2004adversarial,lowd2005adversarial,lowd2005good}, and later \citep{biggio2013evasion,szegedy2013intriguing}, the latter having coined the term.
As presented in \cite{biggio2017wild}, two main settings exist for adversarial machine learning: \emph{evasion}  attacks, where an adversary perturbs data at test time, and \emph{poisoning} attacks, where the data is modified at training time.

The majority of the guarantees and impossibility results for evasion attacks are based on the existence of adversarial examples, potentially crafted by an all-powerful adversary. 
However, what is considered to be an adversarial example has been defined in two different, and in some respects contradictory, ways in the literature. 
The \emph{exact-in-the-ball} notion of robustness (also known as
\emph{error region} risk in \cite{diochnos2018adversarial}) requires that the hypothesis and the ground truth agree in the perturbation region around each test point; the ground truth must thus be specified on all
input points in the perturbation region.
On the other hand, the
constant-in-the-ball notion of robustness (which is also known as
\emph{corrupted input} robustness from the work of \cite{feige2015learning}) requires that the
unperturbed point be correctly classified and that the points in the perturbation region share its label, meaning that we only need access to the test point labels; see, e.g., \citep{diochnos2018adversarial,dreossi2019formalization,gourdeau2021hardness,pydi2021many} for thorough discussions on the subject.

We study the problem of robust classification against evasion attacks under the exact-in-the-ball definition of robustness.
Previous work for this problem, e.g., \citep{diochnos2020lower,gourdeau2021hardness}, has considered the  setting where the learner only has access to random examples. 
However, many defences against evasion attacks have used adversarial training, the practice by which a dataset is augmented with previously misclassified points. 
Moreover, in the learning theory literature, some learning models give more power to the learner, e.g., by using membership and equivalence queries.
Our work studies the robust learning problem mentioned above from a learning theory point of view, and investigates  the power of \emph{local queries} in this setting.

\subsection{Our Contributions}

We outline our contributions below. 
All our results use the \emph{exact-in-the-ball} definition of robustness.
Conceptually, we study the powers and limitations of robust learning with access to oracles that only reveal information nearby the training sample.
Our results are particularly relevant as they contrast with the impossibility of robust learning in the \emph{distribution-free} setting when only random examples are given, as demonstrated in \cite{gourdeau2019hardness}.

\paragraph*{Limitations of the Local Membership Query Model.} In the local membership query ($\LMQ$) model, the learner is allowed to query the label of points in the vicinity of the training sample.
This model was introduced by \cite{awasthi2013learning} and shown to guarantee the PAC learnability of  various concept classes (which are believed or known to be hard to learn with only random examples) under distributional assumptions. 
However, we show that LMQs do not improve the robustness threshold of the class of conjunctions under the uniform distribution.
Indeed, any $\rho$-robust learning algorithm will need a joint sample and query complexity that is exponential in $\rho$, and thus superpolynomial in the input dimension $n$ against an adversary that can flip $\rho = \omega(\log n)$ input bits. 

\paragraph*{The Local Equivalence Query Model.}
Faced with the query lower bound for the $\LMQ$ above, one may consider giving a different power to the learner to improve robust learning guarantees. 
We thus introduce the local equivalence query ($\LEQ$) model, where the learner is allowed to query whether a hypothesis and the ground truth agree in the vicinity of points in the training sample.
The $\LEQ$ oracle is the natural exact-in-the-ball analogue of the Perfect Attack Oracle introduced in \cite{montasser2021adversarially}, which was developed for the constant-in-the-ball robustness.
It is also a variant of Angluin's equivalence query oracle \citep{angluin1987learning}.

\paragraph*{Distribution-Free Robust ERM with an $\LEQ$ Oracle.}
We show that having access to a \emph{robustly} consistent learner (i.e., one that can get zero robust risk on the training sample) gives sample complexity upper bounds that are logarithmic in the size of the hypothesis class or linear in the VC dimension of the robust loss--a complexity measure akin to the adversarial VC dimension of \cite{cullina2018pac}, which we adapted for our notion of robustness.
We study the setting where the learner has access to random examples and an $\LEQ$ oracle.
In the case where the query radius $\lambda$ of the $\LEQ$ oracle is strictly smaller than the adversarial perturbation budget $\rho$, we show that, for a wide variety of concept classes, distribution-free robust learning is impossible, regardless of the training sample size. 
In contrast, when $\lambda=\rho$ we exhibit robustly consistent learners that use an $\LEQ$ oracle.
This separation result further validates the need for an $\LEQ$ oracle in the distribution-free setting.
We furthermore use online learning setting results to exhibit upper bounds on the $\LEQ$ oracle query complexity and then improve these bounds in the specific case of conjunctions.
Finally, we study the sample and query complexity of halfspaces on $\boolhc$, giving robustness guarantees in this case.
For halfspaces in $\R^n$, we obtain sample complexity bounds for linear classifiers against bounded $\ell_2$-norm adversaries. 
To obtain query complexity bounds, we more generally consider adversaries with \emph{bounded precision}, and exhibit robust learning algorithms for this setting.
To the best of our knowledge, the results presented in this paper feature the first robust empirical risk minimization (ERM) algorithms for the \emph{exact-in-the-ball} robust risk in the literature.\footnote{Note that previous work, e.g., \cite{gourdeau2021hardness}, used PAC learning algorithms as black boxes, which are not in general robust risk minimizers, unless they also happen to be exact learning algorithms, and that \citep{montasser2019vc,montasser2021adversarially} use the constant-in-the-ball definition of robustness.}

\subsection{Related Work}

\paragraph*{Learning with Membership and Equivalence Queries.} Membership and equivalence queries (MQ and EQ, respectively) have been widely used in learning theory. 
Membership queries allow the learner to query the label of any point in the input space $\X$, namely, if the target concept is $c$, $\MQ$ returns $c(x)$ when queried with $x\in\X$.
Recall that, in the probably approximately correct ($\PAC$) learning model of \citet{valiant1984theory}, the learner has access to the example oracle $\EX(c,D)$, which upon being queried returns a point $x\sim D$ sampled from the underlying distribution and its label $c(x)$, and the goal is to output $h$ such that with high probability $h$ has low error.\footnote{This is known as the realizable setting. It is also possible to have an arbitrary joint distribution over the examples and labels, in which case we are working in the \emph{agnostic} setting.}
The $\EQ$ oracle takes as input a hypothesis $h$ and returns whether $h=c$, and provides a counterexample $z$ such that $h(z)\neq c(z)$ when $c\neq h$.
The seminal work of \citet{angluin1987learning} showed that deterministic finite automata (DFA) are exactly learnable with a polynomial number of queries to $\MQ$ and $\EQ$ in the size of the DFA. 
Many classes were then shown to be learnable in this setting as well as  others, see e.g., 
\citep{bshouty1993exact,angluin1988queries,jackson1997efficient}.
Moreover, the $\MQ + \EQ$ model has recently been used to extract weighted automata from recurrent neural networks \citep{weiss2018extracting,weiss2019learning,okudono2020weighted}, for binarized neural network verification \citep{shih2019verifying}, and interpretability \citep{camacho2019learning}.
But even these powerful learning models have limitations: learning DFAs only with $\EQ$ is hard \citep{angluin1990negative} and, under cryptographic assumptions, they are also hard to learn solely with the $\MQ$ oracle \citep{angluin1995when}.
It is also worth noting that the $\MQ$ learning model has been criticized by the applied machine learning community, as labels can be queried in the whole input space, irrespective of the distribution that generates the data.
In particular, \citep{baum1992query} observed that query points generated by a learning algorithm on the handwritten characters often appeared meaningless to human labelers.
\citet{awasthi2013learning} thus offered an alternative learning model to Valiant's original model, the PAC and local membership query ($\EX+\LMQ$) model, where the learning algorithm is only allowed to query the label of points that are close to examples from the training sample.
\citet{bary2020distribution} later showed that many concept classes, including DFAs, remain hard to learn in the $\EX+\LMQ$. 

\paragraph*{Existence of Adversarial Examples.}
It has been shown that, in many instances, the vulnerability of learning models to adversarial examples is inevitable due to the nature of the learning problem.
The majority of the results have been shown for the constant-in-the-ball notion of robustness, see e.g., \citep{fawzi2016robustness,fawzi2018adversarial,fawzi2018analysis,gilmer2018adversarial,shafahi2018adversarial,tsipras2019robustness}.
As for the exact-in-the-ball definition of robustness, \citet{diochnos2018adversarial} consider the robustness of monotone conjunctions under the uniform distribution. 
Using the isoperimetric inequality for the boolean hypercube, they show that an adversary that can perturb up to $O(\sqrt n)$ bits can increase the misclassification error from 0.01 to $1/2$. 
\citet{mahloujifar2019curse} then generalize this result to Normal Lévy families and a class of well-behaved classification problems (i.e., ones where the error regions are measurable and average distances exist).

\paragraph*{Sample Complexity of Robust Learning.}
Our work uses a similar approach to \citet{cullina2018pac}, who define the notion of adversarial VC dimension to derive sample complexity upper bounds for robust ERM algorithms, with respect to the constant-in-the-ball robust risk. 
\citet{montasser2019vc} use the same notion of robustness and show sample complexity upper bounds for robust ERM algorithms that are polynomial in the VC and dual VC dimensions of concept classes, giving general upper bounds that are exponential in the VC dimension--though they sometimes must be achieved by an improper learner. 
\citet{ashtiani2020black} build on their work and delineate when proper robust learning is possible. 
On the other hand, \citep{khim2019adversarial,yin2019rademacher,awasthi2020adversarial} study \emph{adversarial} Rademacher complexity bounds for robust learning, giving results for linear classifiers and neural networks when the robust risk can be minimized (in practice, this is approximated with adversarial training).
\citet{viallard2021pac} derive PAC-Bayesian generalization bounds for the averaged risk on the perturbations, rather than working in a worst-case scenario.
As for the exact-in-the-ball definition of robustness,  \citet{diochnos2020lower} show that, for a wide family of concept classes, any learning algorithm that is robust against all $\rho=o(n)$ attacks must have a sample complexity that is at least an exponential in the input dimension $n$. 
They also show a superpolynomial lower bound in case $\rho=\Theta( \sqrt n)$.
\citet{gourdeau2019hardness} show that distribution-free robust learning is generally impossible.
They also show that monotone conjunctions have a robustness threshold of $\Theta(\log n)$ under log-Lipschitz distributions,\footnote{A distribution is log-Lipschitz if the logarithm of the density function is $\log(\alpha)$-Lipschitz w.r.t. the Hamming distance.} meaning that this class is efficiently robustly learnable  against an adversary that can perturb $\log n$ bits of the input, but if an adversary is allowed to perturb $\rho = \omega(\log n)$ bits of the input, there does not exist a sample-efficient learning algorithm for this problem.
\citet{gourdeau2021hardness} extended this result to the class of monotone decision lists and \citet{gourdeau2022sample} showed a sample complexity lower bound for monotone conjunctions that is exponential in $\rho$ and that the robustness threshold of decision lists is also $\Theta(\log n)$.
Finally, \citet{diakonikolas2020complexity} and \citet{bhattacharjee2021sample} have used online learning algorithms for robust learning with respect to the constant-in-the-ball notion of robustness.

\paragraph*{Restricting the Power of the Learner and the Adversary.}
Most adversarial learning guarantees and impossibility results in the literature have focused on all-powerful adversaries.
Recent work has studied learning problems where the adversary's power is curtailed, e.g, \citet{mahloujifar2019can} and \citet{garg2020adversarially} study the robustness of classifiers to polynomial-time attacks. 
Closest to our work, \cite{montasser2020reducing,montasser2021adversarially} study the sample and query complexity of robust learning with respect to the constant-in-the-ball robust risk when the learner has access to a Perfect Attack Oracle (PAO).
For a perturbation type $\U:\X\rightarrow 2^\X$, hypothesis $h$ and labelled point $(x,y)$, the PAO returns the constant-in-the-ball robust loss of $h$ in the perturbation region $\U(x)$ and a counterexample $z$ where $h(z)\neq y$ if it exists.
Our $\LEQ$ oracle is the natural analogue of the PAO oracle for our notion of robustness.
In the constant-in-the-ball \emph{realizable} setting,\footnote{\label{note:realizable-c-i-b}In the realizable setting, there exists a hypothesis that has zero constant-in-the-ball robust loss.} the authors use online learning results to show sample and query complexity bounds that are linear and quadratic in the Littlestone dimension of concept classes, respectively \citep{montasser2020reducing}.
\citet{montasser2021adversarially} moreover use the algorithm from \citep{montasser2019vc} to get a sample complexity of $\tilde{O}\left(\frac{\VC(\mathcal{H}){\VC^*}^2(\mathcal{H})+\log(1/\delta)}{\epsilon}\right)$ and query complexity of $\tilde{O}(2^{\VCH^2 \DVCH^2\log^2(\DVCH)}\Lit(\mathcal{H}))$, where $\DVCH$ is the dual VC dimension of the hypothesis class $\H$.
Finally, they extend their results to the agnostic setting and derive lower bounds. 
As in the setting with having only access to the example oracle,  different notions of robustness have vastly different implications in terms of robust learnability of certain concept classes.
Whenever relevant, we will draw a thorough comparison in the next sections between our work and that of \citet{montasser2021adversarially}.

\section{Problem Set-Up}

We work in the PAC learning framework (see Appendix~\ref{app:pac}), with the distinction that a robust risk function is used instead of the standard risk.
We will study metric spaces $(\X_n,d)$ of input dimension $n$ with a perturbation budget function $\rho:\N\rightarrow\R$ defining the perturbation region $B_\rho(x):=\set{z\in\X_n\given d(x,z)\leq\rho(n)}$. 
When the input space is the boolean hypercube $\X_n=\boolhc$, the metric is the Hamming distance.

We use the exact-in-the-ball robust risk, which is defined w.r.t. a hypothesis $h$, target $c$ and distribution $D$ as the probability $\risk_\rho^D(h,c):=\Prob{x\sim D}{\exists z\in B_\rho(x) \st c(z) \neq h(z)}$ that $h$ and $c$ disagree in the perturbation region.
On the other hand, the constant-in-the-ball robust risk is defined as $\Prob{x\sim D}{\exists z\in B_\rho(x) \st c(x) \neq h(z)}$.
Note that it is possible to adapt the latter to a joint distribution on the input and label spaces, but that there is an implicit \emph{realizability assumption} in the former as the prediction on perturbed points' labels are compared to the ground truth $c$.
We emphasize that choosing a robust risk function should depend on the learning problem at hand.
The constant-in-the-ball notion of robustness requires a certain form of \emph{stability}: the hypothesis should be correct on a random example and not change label in the perturbation region; this robust risk function may be more appropriate in settings with a strong margin assumption. 
In contrast, the exact-in-the-ball notion of robustness speaks to the \emph{fidelity} of the hypothesis to the ground truth, and may be more suitable when a considerable portion of the probability mass is in the vicinity of the decision boundary.
\citet{diochnos2018adversarial,dreossi2019formalization,gourdeau2021hardness,pydi2021many} offer a thorough comparison between different notions of robustness.

In the face of the impossibility or hardness of robustly learning certain concept classes, either through statistical or computational limitations, it is natural to study whether these issues can be circumvented by giving more power to the learner.
The $\lambda$-local membership query ($\lambda$-$\LMQ$) set up of \cite{awasthi2013learning}, which is formally defined in Appendix~\ref{app:lmq}, allows the learner to query the label of points that are at distance at most $\lambda$ from a sample $S$ drawn randomly from $D$.
Inspired by this learning model, we define  the $\lambda$-local equivalence query ($\lambda$-$\LEQ$) model where, for a point $x$  in a sample $S$ drawn from the underlying distribution $D$, the learner is allowed to query an oracle that returns whether $h$ agrees with the ground truth $c$ in the ball $B_\lambda(x)$ of radius $\lambda$ around $x$.\footnote{Similarly to $\rho$, we implicitly consider $\lambda$ as a function of the input dimension $n$. It is also possible to extend this definition to an arbitrary perturbation function $\U:\X\rightarrow 2^\X$.} 
If they disagree, a counterexample in $B_\lambda(x)$ is returned as well.
Clearly, by setting $\lambda=n$, we recover the $\EQ$ oracle.\footnote{This is evidently not the case for the Perfect Attack Oracle of \cite{montasser2021adversarially}.}
Note moreover that when $\lambda=\rho$, this is equivalent to querying the (exact-in-the-ball) robust loss around a point.
We will show a separation result for robust learning algorithms between models that only allow random examples and ones that allow random examples and access to $\LEQ$.

\begin{definition}[$\lambda$-$\LEQ$ Robust Learning]
\label{def:leq}
Let $\X_n$ be the instance space, $\C$ a concept class over $\X_n$, and $\D$ a class of distributions over $\X_n$. We say that $\C$ is $\rho$-robustly learnable using $\lambda$-local equivalence queries with respect to distribution class, $\D$, if there exists a learning algorithm, $\A$, such that for every $\epsilon > 0$, $\delta > 0$, for every distribution $D\in\D$ and every target concept $c\in\C$, the following hold:\footnote{We implicitly assume that a concept $c\in\C$ can be represented in size polynomial in $n$, where $n$ is the input dimension; otherwise a parameter $size(c)$ can be introduced in the sample and query complexity requirements.}
\begin{enumerate}
\item $\A$ draws a sample $S$ of size $m = \poly(n, 1/\delta, 1/\epsilon)$ using the example oracle $\EX (c, D)$
\item Each query made by $\A$ at $x \in S$ and for a candidate hypothesis $h$ to $\lambda$-$\LEQ$ either confirms that $c$ and $h$ coincide on $B_\lambda(x)$ or returns $z\in B_\lambda(x)$ such that $c(z) \neq h(z)$. $\A$ is allowed to update $h$ after seeing a counterexample
\item  $\A$ outputs a hypothesis $h$ that satisfies $\risk_\rho^D(h,c)\leq \epsilon$ with probability at least $1-\delta$ 
\item The running time of $\A$ (hence also the number of oracle accesses) is polynomial in $n$, $1/\epsilon$, $1/\delta$ and the output hypothesis $h$ is polynomially evaluable.
\end{enumerate}
\end{definition}

We remark that this model evokes the online learning setting, where the learner receives counterexamples after making a prediction, but with a few key differences. 
Contrary to the online setting (and the exact learning framework with $\MQ$ and $\EQ$), there is an underlying distribution with which the performance of the hypothesis is evaluated in both the $\LMQ$ and $\LEQ$ models.
Moreover, in online learning, when receiving a counterexample, the only requirement is that there is a concept that correctly classifies all the data given to the learner up until that point, and so the counterexamples can be given in an \emph{adversarial} fashion, in order to maximize the regret. 
However, both the $\LMQ$ and $\LEQ$ models require that a target concept be chosen a priori.
Note though that the $\LEQ$ oracle can give any counterexample for the robust loss at a given point.

In practice, one always has to find a way to approximately implement oracles studied in theory.
A possible way to generate counterexamples with respect to the exact-in-the-ball notion of robustness is as follows.
Suppose that there is an adversary that can generate points $z\in B_\rho(x)$ such that $h(z)\neq c(z)$.
Provided such an adversary can be simulated, there is a way to (imperfectly) implement the $\LEQ$ oracle in practice. 
Thus, the use of these oracles can be viewed as a form of adversarial training.

Both the $\LMQ$ and $\LEQ$ models are particularly well-suited for the standard and exact-in-the-ball risks, as they address \emph{information-theoretic} limitations of learning with random examples only. 
 On the other hand, while information-theoretic limitations of robust learning with respect to the \emph{constant-in-the-ball} notion of robustness arise when the perturbation function $\U$ is unknown to the learner, \emph{computational} obstacles can also occur even when the definition of $\U$ is available. 
Indeed, determining whether the hypothesis changes label in the perturbation region could  be intractable.
In these cases, the Perfect Attack Oracle of \cite{montasser2021adversarially} can be used to remedy these limitations for robust learning with respect to the constant-in-the-ball robust risk.
Crucially, in their setting, counterexamples could have a different label to the ground truth: a counterexample $z\in\U(x)$ for $x$ is such that $h(z)\neq c(x)$, not necessarily $h(z)\neq c(z)$. This could compromise the standard accuracy of the hypothesis (see e.g., \cite{tsipras2019robustness} for a learning problem where robustness and accuracy are at odds).
Finally, an $\LMQ$ analogue for the constant-in-the-ball risk is not needed: the only information we need for a perturbed point $z\in B_\rho(x)$  is the label of $x$ (given by the example oracle) and $h(z)$. 
Given that one of the requirements of PAC learning is that the hypothesis is efficiently evaluatable, we can easily compute $h(z)$.

\section{Distribution-Free Robust Learning with Local Equivalence Queries}
\label{sec:df-rob}

In this section, we show that having access to a local equivalence query oracle can guarantee the efficient \emph{distribution-free} robust learnability of certain concept classes. 
We start with a negative result which shows that for a wide variety of concept classes, if $\lambda<\rho$, then \emph{distribution-free} robust learnability is impossible with $\EX+\lambda$-$\LEQ$ -- regardless of how many queries are allowed. 
However, the regime $\lambda=\rho$, which implies giving similar power to the learner as the adversary, enables robust learnability guarantees.
Indeed, Section~\ref{sec:sample-compl} exhibits upper bounds on sample sizes that will guarantee \emph{robust} generalization. 
These bounds are logarithmic in the size of the hypothesis class (finite case) and linear in the \emph{robust} VC dimension of a concept class (infinite case). 
Section~\ref{sec:online} draws a comparison between our framework and the online learning setting, and exhibits robustly consistent learners. 
Section~\ref{sec:conj} studies the class of conjunctions and presents a robust learning algorithm that is \emph{both} statistically and computationally efficient. 
Finally, Section~\ref{sec:halfspaces} looks at linear classifiers in the discrete and continuous cases, and adapts the Winnow and Perceptron algorithms to both settings. 

\subsection{Impossibility of Distribution-Free Robust Learning When  $\lambda<\rho$}

We start with a negative result, saying that whenever the local query radius is strictly smaller than the adversary's budget, monotone conjunctions are not distribution-free robustly learnable, which is in contrast to the standard PAC setting where guarantees hold \emph{for any distribution}.
Note that our result goes beyond efficiency: no query can distinguish between two potential targets. 
Choosing the target uniformly at random lower bounds the expected robust risk, and hence renders robust learning impossible in this setting. 

\begin{theorem}
\label{thm:mon-conj-df-leq}
For locality and robustness parameters $\lambda,\rho\in\N$ with $\lambda < \rho $, monotone conjunctions (and any superclass) are not distribution-free $\rho$-robustly learnable with access to a $\lambda$-$\LEQ$ oracle.
\end{theorem}

\begin{proof}
Fix $\lambda,\rho\in\N$ such that $\lambda < \rho $, and consider the following monotone conjunctions: $c_1(x)=\bigwedge_{1 \leq i \leq \rho} x_i$ and $c_2(x)=\bigwedge_{1 \leq i \leq \rho +1} x_i$.
Let $D$ be the distribution on $\boolhc$ which puts all the mass on $\mathbf{0}$.
Then, the target concept is drawn at random between $c_1$ and $c_2$. 
Now, $c_1$ and $c_2$ will both give all points in $B_\lambda(\mathbf{0})$ the label 0, so the learner has to choose a hypothesis that is consistent with both $c_1$ and $c_2$ (otherwise the robust risk is 1 and we are done). 
However, the learner has no way of distinguishing which of $c_1$ or $c_2$ is the target concept, while these two functions have a $\rho$-robust risk of 1 against each other under $D$.
Formally, 
\begin{align}
\risk_\rho^D(c_1,c_2) 
&= \Prob{x\sim D}{\exists z\in B_\rho(x)\st c_1(z)\neq c_2(z)}\notag\\
&= \mathbf{1}[\exists z\in B_\rho(\mathbf{0})\st c_1(z)\neq c_2(z)]\notag\\ 
&=1 \label{eqn:mon-conj-risk-1}
\enspace,
\end{align}
where such $z=\mathbf{1}_\rho\mathbf{0}_{n-\rho}$.
To lower bound the expected robust risk, letting $\A$ be any learning algorithm and $\E$ be the event that all points in a randomly drawn sample $S$ are all labeled 0, we have
\begin{align*}
\eval{c,S}{\risk_\rho^D(\A(S),c)}
&= \eval{c,S}{\risk_\rho^D(\A(S),c)\given \E} \tag{By construction of $D$}
 \\
&= \frac{1}{2}\;\eval{S}{\risk_\rho^D(\A(S),c_1)+\risk_\rho^D(\A(S),c_2)\given \E} 
\tag{Random choice of $c$} \\
&\geq\frac{1}{2}\;\eval{S}{\risk_\rho^D(c_1,c_2)\given \E} \tag{Lemma~\ref{lemma:rob-triangle}}\\
&=\frac{1}{2} \tag{Equation~\ref{eqn:mon-conj-risk-1}}
\enspace.
\end{align*}
\end{proof}

The result holds for monotone conjunctions and all superclasses (e.g., decision lists and halfspaces), but, in fact, we can generalize this reasoning for any concept class that has a certain form of stability: 
if we can find concepts $c_1$ and $c_2$ in $\C$ and points $x,x'\in \X$ such that $c_1$ and $c_2$ agree on $B_\lambda(x)$ but disagree on $x'$, then if $\lambda < \rho$, the concept class $\C$ is not distribution-free $\rho$-robustly learnable with access to a $\lambda$-$\LEQ$ oracle.
It suffices to ``move'' the center of the ball $x$ until we find a point in the set $B_\rho(x)\setminus B_\lambda(x)$ where $c_1$ and $c_2$ disagree, which is guaranteed to happen by the existence of $x'$.

\subsection{General Sample Complexity Bounds for Robustly-Consistent Learners}
\label{sec:sample-compl}

In this section, we show that we can derive sample complexity upper bounds for \emph{robustly} consistent learners, i.e., learning algorithms that return a \emph{robust} loss of zero on a training sample.
Note that, crucially,  the exact-in-the-ball notion of robustness and its realizability imply that any robust ERM algorithm will achieve zero empirical robust loss on a given training sample.
As we will see in the next sections, the challenge is to find a \emph{robustly} consistent learning algorithm that uses queries to $\rho$-$\LEQ$.
The first bound is for finite classes, where the dependency is logarithmic in the size of the hypothesis class. 
The proof is a simple application of Occam's razor and is included in Appendix~\ref{app:sample-compl} for completeness.
The argument is similar to \cite{bubeck2019adversarial}.

\begin{lemma}
\label{lemma:occam}
Let $\C$ be a concept class and $\mathcal{H}$ a hypothesis class.
Any $\rho$-robust ERM algorithm using $\mathcal{H}\supseteq\C$ on a sample of size $m\geq \frac{1}{\epsilon}\left(\log |\mathcal{H}_n|+\log\frac{1}{\delta}\right)$ is a $\rho$-robust learner for  $\C$.
\end{lemma}

\begin{proof}
Fix a target concept $c\in\C$ and the target distribution $D$ over $\X$. 
Define a hypothesis $h$ to be ``bad'' if $R_\rho^D(c,h) \geq \epsilon$. 
Note that any robust ERM algorithm will be robustly consistent on the training sample by the realizability assumption.
Let $\E_h$ be the event that $m$ independent examples drawn from $\EX(c, D)$ are all robustly consistent with $h$. 
Then, if $h$ is bad, we have that $\Prob{}{\E_h}\leq (1-\epsilon)^m\leq e^{-\epsilon m}$.
Now consider the event $\E = \bigcup_{h\in\mathcal{H}}\E_h$.
We have that, by the union bound, $$\Prob{}{\E} \leq \sum_{h\in\mathcal{H}}\Prob{}{\E_h}\leq \abs{\mathcal{H}}e^{-\epsilon m}\enspace.$$
Then, bounding the RHS by $\delta$, we have that whenever $m\geq \frac{1}{\epsilon}\left(\log |\mathcal{H}_n|+\log\frac{1}{\delta}\right)$, no bad hypothesis is \emph{robustly} consistent with $m$ random examples drawn from $\EX(c, D)$. 
If a hypothesis is not bad, it has robust risk bounded above by $\epsilon$, as required.
\end{proof}

For the infinite case, we cannot immediately use the VC dimension as a tool for bounding the sample complexity of robust learning.
To this end, we use the VC dimension of the robust loss between two concepts, which is the VC dimension of the class of functions representing the $\rho$-expansion of the error region between any possible target and hypothesis.
This is analogous to the adversarial VC dimension defined by \cite{cullina2018pac} for the constant-in-the-ball definition of robustness. 

\begin{definition}[VC dimension of the exact-in-the-ball robust loss]
\label{def:rob-vc}
Given a target concept class $\C$, a hypothesis class $\mathcal{H}$ and a robustness parameter $\rho$, the VC dimension of the robust loss between $\C$ and $\H$ is defined as $\RVClong$, where $\RLossLong=\set{\Rloss: x\mapsto \mathbf{1}[\exists z\in B_\rho(x)\st c(z) \neq h(z)] \given c\in\C, h\in\mathcal{H}}$.
Whenever $\C=\mathcal{H}$, we simply write $\RVC$.
\end{definition}

We now show that we can use the VC dimension of the robust loss to upper bound the sample complexity of robustly-consistent learning algorithms. 

\begin{lemma}
\label{lemma:rob-vc}
Let $\C$ be a concept class and $\mathcal{H}$ a hypothesis class. Any $\rho$-robust ERM algorithm using $\mathcal{H}$ on a sample of size $m\geq \frac{\kappa_0}{\epsilon}\left(\RVClong\log(1/\epsilon)+\log\frac{1}{\delta}\right)$ is a $\rho$-robust learner for  $\C$.
\end{lemma}
%

\begin{proof}
The proof is very similar to the VC dimension upper bound in PAC learning.
The main distinction is that instead of looking at the error region of the target and any function in $\mathcal{H}$, we must look at its $\rho$-expansion.
Namely, we let the target $c\in\C$ be fixed and, for $h\in\mathcal{H}$, we consider the function $(c\oplus h)_\rho: x\mapsto \mathbf{1}[\exists z\in B_\rho(x)\st c(z) \neq h(z)]$ and define a new concept class $\Delta_{c,\rho}(\mathcal{H})=\set{(c\oplus h)_\rho \given h\in\mathcal{H}}$.
As we have that $\Delta_{c,\rho}(\mathcal{H})\subseteq\RLossLong$, it follows that $\VC(\Delta_{c,\rho}(\mathcal{H}))\leq \RVClong$ (any sign pattern achieved on the LHS can be achieved on the RHS).

The remainder of the proof follows from the definition of an $\epsilon$-net and the bound on the growth function of $\Delta_{c,\rho}(\mathcal{H})$.

First, define the class $\Delta_{c,\rho,\epsilon}(\mathcal{H})$ as $\set{\tilde{c}\in\Delta_{c,\rho}(\mathcal{H})\given \Prob{x\sim D}{\tilde{c}(x)=1}\geq \epsilon}$, i.e., the set of functions in $\Delta_{c,\rho}(\mathcal{H})$ which have a robust risk greater than $\epsilon$. 
Recall that a set $S$ is an $\epsilon$-net for $\Delta_{c,\rho}(\mathcal{H})$ if for every $\tilde{c}\in \Delta_{c,\rho,\epsilon}(\mathcal{H})$, there exists $x\in S$ such that $\tilde{c}(x)=1$. 
We want to bound the probability that a sample $S\sim D^m$ fails to be an $\epsilon$-net for the class $\Delta_{c,\rho}(\mathcal{H})$, as if $S$ is an $\epsilon$-net, then any robustly consistent $h\in\mathcal{H}$ on $S$ will have robust risk bounded above by $\epsilon$.
As with the standard VC dimension, a sample $S$ will be drawn in two phases.
First draw a sample $S_1\sim D^m$ and let $\E_1$ be the event that $S_1$ is not an $\epsilon$-net for $\Delta_{c,\rho}(\mathcal{H})$. 
Now, suppose $\E_1$ occurs.
This means there exists $\tilde{c} \in \Delta_{c,\rho,\epsilon}(\mathcal{H})$ such that $\tilde{c}(x)=0$ for all the points $x\in S_1$.
Fix such a $\tilde{c}$ and draw a second sample $S_2\sim D^m$.
Then, letting $X$ be the random variable representing the number of points in $S_2$ that are such that $\tilde{c}(x)=1$, we can use Chernoff bound to show that
\begin{equation}
\label{eqn:chernoff-lb}
\Prob{}{X<\epsilon m /2}
\leq 2\exp \left(-\frac{\epsilon m}{12}\right)
\enspace,
\end{equation}
ensuring that whenever $\epsilon m \geq 24$, the probability that at least $\epsilon m/2$ points in $S_2$ satisfy $\tilde{c}(x)=1$ is bounded below by $1/2$.

Now, consider the event $\E_2$ where a sample $S=S_1 \cup S_2$ of size $2m$ such that $|S_1|=|S_2|=m$ is drawn from $\EX(c,D)$ and there exists a concept $\tilde{c}\in \Pi_{\Delta_{c,\rho,\epsilon}(\mathcal{H})}(S)$ such that $|\set{x\in S\given \tilde{c}(x)=1|\geq \epsilon m /2}$ and $\tilde{c}(x)=0$ for all $x\in S_1$, where $\Pi_{\Delta_{c,\rho,\epsilon}(\mathcal{H})}(S)$ is the set all possible dichotomies on $S$ induced by $\Delta_{c,\rho,\epsilon}(\mathcal{H})$.
Then $\Prob{}{\E_2}\geq \frac{1}{2}\Prob{}{\E_1}$ from Equation~\ref{eqn:chernoff-lb}.
Now, the probability that $\E_2$ happens for a fixed $\tilde{c}\in\Delta_{c,\rho,\epsilon}(\mathcal{H})$ is 
\begin{equation*}
\frac{{m\choose \epsilon m/2}}{{2m\choose \epsilon m/2}}
\leq 2^{-\epsilon m/2}
\enspace.
\end{equation*}
Finally, letting $d=\RVClong$ we can bound the probability of $\E_1$ using the union bound:
\begin{align*}
\Prob{}{\E_1}&\leq 2\Prob{}{\E_2} \\
&\leq 2 \abs{\Pi_{\Delta_{c,\rho,\epsilon}(\mathcal{H})}(S)} 2^{-\epsilon m/2}\\
&\leq 2 \abs{\Pi_{\Delta_{c,\rho}(\mathcal{H})}(S)} 2^{-\epsilon m/2}\\
&\leq 2 \left(\frac{2em}{d}\right)^d 2^{-\epsilon m/2} 
\enspace,
\end{align*}
where the last inequality is due to the Sauer-Shelah Lemma \citep{sauer1972density,shelah1972combinatorial}.
Thus, there exists a universal constant such that provided $m$ is larger than the bound given in the statement of the theorem, $\Prob{}{\E_1}<\delta$, as required.
\end{proof}

\begin{remark}
\label{rmk:rho-tradeoff}
For the boolean hypercube and the Hamming distance, note that, as $\rho(n)/n$ tends to $1$, we move towards the  exact and online learning settings, and the underlying distribution becomes less important. 
In this case, the VC dimension of the robust loss starts to decrease.
Indeed, say if $\rho=n$, then $(\C\oplus \C)_\rho$ only contains the constant functions $0$ and $1$. 
We thus only need a single example to query the $\LEQ$ oracle (which has become the $\EQ$ oracle).
However, this comes at a cost: the \emph{query complexity} upper bounds presented in the next sections could be tight.
Understanding the behaviour of the VC dimension of the robust loss as a function of $\rho$ and deriving joint sample and query complexity bounds are both avenues for future research.
\end{remark}

\subsection{Query Complexity Bounds Using Online Learning Results}
\label{sec:online}

In the previous section, we derived sample complexity upper bounds for robustly consistent learners.
The challenge is thus to create algorithms that perform robust empirical risk minimization, as we are operating in the realizable setting.
We begin by showing that, if one can ignore computational limitations, then online learning results can be used to guarantee robust learnability. 
We recall the online learning setting in Appendix~\ref{app:online}.
We denote by $\Lit(\C)$ the Littlestone dimension of a concept class $\C$, which is defined in Appendix~\ref{app:complexity} and appears in the query complexity bound in the theorem below.

\begin{theorem}
\label{thm:soa}
A concept class $\C$ is $\rho$-robustly learnable with the Standard Optimal Algorithm (SOA)  \citep{littlestone1988learning} using the $\EX$ and $\rho$-$\LEQ$ oracles with sample complexity $ m(n,\epsilon,\delta) = \frac{\kappa_0}{\epsilon}\left(\RVC\log(1/\epsilon)+\log\frac{1}{\delta}\right)$ for a sufficiently large constant $\kappa_0$, and query complexity $r (n,\epsilon,\delta) = m(n,\epsilon,\delta)\cdot \Lit(\C)$.
Furthermore, if $\C$ is a finite concept class on $\boolhc$, then $\C$ is $\rho$-robustly learnable with sample complexity $ m (n,\epsilon,\delta) = \frac{1}{\epsilon}\left(\log(|\C|)+\log\frac{1}{\delta}\right)$ and query complexity $r (n,\epsilon,\delta) = m(n,\epsilon,\delta)\cdot \Lit(\C) $.
\end{theorem}

\begin{proof}
The sample complexity bounds come from Lemmas~\ref{lemma:occam} and~\ref{lemma:rob-vc} and the fact that the Standard Optimal Algorithm (SOA) is a consistent learner, as it will be given counterexamples in the perturbation region until a robust loss of zero is achieved. 

For each query to $\LEQ$, a counterexample is returned, or the robust loss is zero. 
Then, using the mistake upper bound of SOA, which is $\Lit(\C)$, we get the query upper bound.
\end{proof}

Of course, some concept classes, e.g., thresholds, have infinite Littlestone dimension, so our bounds are not useful in these settings. 
In Section~\ref{sec:halfspaces}, we study distributional assumptions that give reasonable query upper bounds for linear classifiers, using the theorem below.
It exhibits a query upper bound for robustly learning with an online algorithm $\A$ with a given mistake upper bound $M$.
This is moreover particularly useful in case $M$ is polynomial in the input dimension and $\A$ is \emph{computationally} efficient (which is not the case for the Standard Optimal Algorithm in Theorem~\ref{thm:soa}).

\begin{lemma}
\label{lemma:rob-mistake-bound}
Fix $\rho>0$. 
Let $\C$ be a concept class and $\D=\set{\D_c}_{c\in\C}$ a distribution family, where $\D_c$ is a subfamily of distributions.
Suppose that $\C$ is learnable in the online setting with mistake  bound $M(n)$ whenever, for a given target $c$ the instance space is restricted to a given subset $\X_c\subseteq \X$.
Moreover suppose that for any $c\in\C$ and $D\in\D_c$, $B_\rho(\supp(D))\subseteq\X_c$.
Then $\set{(c,D)\given c\in\C,\; D\in\D_c}$ is $\rho$-robustly learnable using the $\EX$ and $\rho$-$\LEQ$ oracles with sample complexity $ m (n,\epsilon,\delta) = \frac{1}{\epsilon}\left(\RVClong+\log\frac{1}{\delta}\right)$ and query complexity $r (n,\epsilon,\delta) = m (n,\epsilon,\delta) \cdot M(n)$.
\end{lemma}

\begin{proof}
The sample complexity bound $m$ is obtained from Lemma~\ref{lemma:rob-vc} and, for each point in the sample, a query to $\LEQ$ can either return a robust loss of 0 or 1 and give a counterexample. 
Since the mistake bound is $M(n)$, and all counterexamples  come from the region $B_\rho(\supp(D))$, we have a query upper bound of $r = m \cdot M(n)$, as required.
\end{proof}

\subsection{Improved Query Complexity Bounds: Conjunctions}
\label{sec:conj}

In this section, we show  how to improve the query upper bound from the previous section in the special case of conjunctions. 
Moreover, the algorithm used to robustly learn conjunctions is both statistically and \emph{computationally} efficient, which is not the case for the Standard Optimal Algorithm.

\begin{theorem}
\label{thm:conj-df-leq}
The class $\Conj$ is efficiently $\rho$-robustly learnable in the distribution-free setting using the $\EX$ and $\rho$-$\LEQ$ oracles with at most $O\left( \frac{1}{\epsilon}\left(n+\log\frac{1}{\delta}\right)\right)$ random examples and $O\left( \frac{1}{\epsilon}\left(n+\log\frac{1}{\delta}\right)\right)$ queries to $\rho$-$\LEQ$.
\end{theorem}

\begin{proof}
Let $c$ be the target conjunction and let $D$ be an arbitrary distribution. 
We describe an algorithm $\A$ with polynomial sample and query complexity with access to a $\rho$-$\LEQ$ oracle.
By Lemma~\ref{lemma:occam}, if we can get guarantee that $\A$ returns a hypothesis with zero robust loss on a i.i.d. sample of size $m= O\left( \frac{1}{\epsilon}\left(n+\log\frac{1}{\delta}\right)\right)$ with a polynomial number of queries to the $\rho$-$\LEQ$ oracle, we are done.

The algorithm  is similar to the standard PAC learning algorithm, in that it only learns from positive examples.
Indeed, the original hypothesis $h$ is a conjunction of all of the $2n$ literals. 
After seeing a positive example $x$, $\A$ removes from $h$ the literals $\bar{x_i}$ for $i=1,\dots,n$, as they cannot be in $c$.
Note that, by construction, any hypothesis $h$ returned by $\A$ always satisfies $c \subseteq h$.\footnote{We overload $c,h$ to mean both the functions and the set of literals in the conjunction, as it will be unambiguous to distinguish them from context.}
Thus, any counter example returned by the $\LEQ$ oracle will have that $c(z)=1$ and $h(z)=0$. 
This allows us to remove at least one literal from the hypothesis set for every counter example.
Now, it is easy to see that, for $c\subseteq h' \subseteq h$, if the robust loss $\mathbf{1}[\exists z \in B_\lambda(x) \st c(z) \neq h(z)]$  on $x$ w.r.t. $h$ is zero, so will be the robust loss on $x$ w.r.t. the updated hypothesis $h'$. 
Hence, $\A$ makes at most $m+2n$ queries to the $\LEQ$ oracle.
\end{proof}

Note that the query upper bound that we get is of the form $m+M$, as opposed to $m\cdot M$ from Lemma~\ref{lemma:rob-vc} (where $m$ is the sample complexity and $M$ the mistake bound).
This is because we have adapted the PAC learning algorithm for conjunctions to our setting. 
Any update to its hypothesis will not affect the consistency of previously queried points with robust loss of zero, and thus once zero robust loss is achieved on a point, it does not need to be queried again.  

\subsection{Linear Classifiers}
\label{sec:halfspaces}

We now derive sample and query complexity upper bounds for restricted subclasses of linear classifiers. 
We start with linear classifiers on $\boolhc$ with bounded weight, and then study linear classifiers on $\R^n$.
Note that the robustness threshold of linear classifiers on $\boolhc$ \emph{without} access to the $\LEQ$ oracle remains an open problem~\citep{gourdeau2022sample}.\footnote{With respect to the exact-in-the-ball definition of robustness.}

Let $\Halfspaces_{\boolhc}^W$  be the class of linear threshold functions on $\boolhc$ with integer weights  such that the sum of the absolute values of the weights and the bias is bounded above by $W$. 
We have the following theorem, which gives both sample and query complexity upper bounds for the robust learnability of $\Halfspaces_{\boolhc}^W$. 

\begin{theorem}
\label{thm:ltf-bool-df}
The class $\Halfspaces_{\boolhc}^W$ is $\rho$-robustly learnable with access to the $\EX$ and $\rho$-$\LEQ$ oracles by using the Winnow algorithm with sample complexity $m(n,\epsilon,\delta)=O\left(\frac{1}{\epsilon}\left(n+\min\set{n,W}\log (W+n)+\log\frac{1}{\delta}\right)\right)$ and query complexity $O(m(n,\epsilon,\delta) \cdot W^2\log(n))$.
\end{theorem}

\begin{proof}
The sample complexity bound uses Lemma~\ref{lemma:occam}.
Note the class $\Halfspaces_{\boolhc}^W$  has size $O(2^n (n+W)^{\min\set{n,W}})$.
This  is a simple application of the stars and bars identity, where $W$ is the number of stars and $n+1$ the number of bars (as we are considering the bias term as well): ${n+W\choose W}=O( (n+W)^{min\set{n,W}})$. 
The $2^n$ term comes from the fact that each weight can be positive or negative.
The query complexity uses the fact that the mistake bound for Winnow for $\Halfspaces_{\boolhc}^W$ is $O(W^2\log(n))$ in the case of positive weights (the full statement can be found in Appendix.~\ref{app:useful}).
\citet{littlestone1988learning} outlines how to use the Winnow algorithm when the linear classifier's weights can vary in sign, at the cost of doubling the input dimension and weight bound (see Theorem 10 and Example 6 therein).
\end{proof}

We now turn our attention to linear classifiers $\ltfreal$ on $\R^n$. 
We first show that, when considering an adversary with bounded $\ell_2$-norm perturbations, we can bound the sample complexity of robust learning for this class through a bound on the VC dimension of the robust loss. 
However, the query complexity is infinite in the general case.
This is because the Littlestone dimension of thresholds, and thus halfspaces, is infinite. 
We will address this issue in Section~\ref{sec:adv-bounded-precision}.

\begin{theorem}
\label{thm:sc-ub-ltf-real}
Let the adversary's budget be measured by the $\ell_2$ norm.
Then any $\rho$-robust ERM learning algorithm for $\ltfreal$ on $\R^n$ has sample complexity $m=O(\frac{1}{\epsilon}( n^3 + \log (1/\delta)))$.
\end{theorem}

The proof of this theorem relies on deriving an upper bound on the VC dimension of the robust loss of halfspaces.
This will help us bound the sample complexity needed to guarantee robust accuracy. 
To bound the VC dimension of the robust loss of linear classifiers, we will need the following theorem from \cite{goldberg1995bounding}:

\begin{theorem}[Theorem 2.2 in \cite{goldberg1995bounding}]
\label{thm:goldberg}
Let $\set{\C_{k,n}}_{k,n\in\N}$ be a family of concept classes where concepts in $\C_{k,n}$ and instances are represented by $k$ and $n$ real values, respectively.
Suppose that the membership test for any instance $\alpha$ in any concept $C$ of $\C_{k,n}$ can be expressed as a boolean formula $\Phi_{k,n}$ containing $s = s(k,n)$ distinct atomic predicates, each predicate being a polynomial inequality or equality over $k+n$ variables (representing $C$ and $\alpha$) of degree at most $d=d(k,n)$.
Then $\VC({\C_{k,n}}) \leq 2k\log (8eds)$. 
\end{theorem}

We will now translate the $\rho$-expansion of the error region (i.e., the robust loss function) between two halfspaces as a boolean formula using a result from \cite{renegar1992computational}.
This will allow us to use the theorem above from \cite{goldberg1995bounding} to bound the VC dimension of the robust loss of $\ltfreal$.

\begin{lemma}
\label{lemma:rob-risk-bool-formula}
Let $a,b\in\R^{n}, a_0,b_0\in\R$, and define the map $\varphi: x \mapsto \mathbf{1}[\exists z\in B_\rho(x) \st \sgn(a^\top z + a_0)\neq \sgn(b^\top z + b_0)]$.
Then $\varphi$ can be represented as a boolean formula $\Phi$ with $s=10^{Cn^2}$ distinct atomic predicates, each predicate being a polynomial inequality over $2n+2$ variables of degree at most $10^{C'n}$ for some constants $C,C'>0$.
\end{lemma}

\begin{proof}
First note that the predicate $\sgn(a^\top z+a_0)\neq\sgn(b^\top z+b_0)$ 
can be represented as the following formula:
\begin{equation*}
\left(a^\top z+a_0 \geq 0 \wedge b^\top z+b_0 < 0\right) \vee \left( a^\top z+a_0 <0 \wedge b^\top z+b_0\geq 0\right)\enspace,
\end{equation*}
which contains $n+(2n+2)$ variables and 4 predicates.
Moreover, given a perturbation $\zeta\in\R^n$, the constraint $\norm{\zeta}_2\leq \rho$ on its magnitude is a polynomial inequality of degree 2:
$$\sum_i \zeta_i^2 \leq \rho^2 \enspace.$$
Now, consider the following formula:
\begin{equation*}
\Psi(x) = \exists \zeta \in\R^n \st \left( \sgn(a^\top (x+\zeta)+a_0)\neq\sgn(b^\top (x+\zeta)+b_0) \wedge \norm{\zeta}_2\leq \rho\right)
\enspace.
\end{equation*}
This is a formula of first-order logic over the reals.
Using the notation of Theorem~\ref{thm:renegar}, we have $\omega=1$ quantifier, and thus $\prod_k n_k = n$, one Boolean formula with $m=5$ polynomial inequalities of degree $d$ at most $2$, and $l=n$.
Thus, $\Psi(x)$ can be expressed as a quantifier-free formula $\Phi(x)=\bigvee_{i=1}^I \bigwedge_{j=1}^{J_i} (h_{ij}(y)\Delta_{ij} 0)$ of size $$I\max_i J_i \leq (md)^{2^{O(\omega)}l\prod_k n_k+2^{O(\omega)}\prod_k n_k}\leq 10^{Cn^2} $$ for some constant $C$, where the polynomial inequalities are of degree at most $(md)^{2^{O(\omega)}\prod_k n_k}\leq 10^{C'n}$ for some constant $C'$.
\end{proof}

We thus get the following corollary.

\begin{corollary}
\label{cor:rvc-ltf}
The VC dimension of the robust loss of $\ltfreal$ 
is $O( n^3)$.
\end{corollary}
\begin{proof}
We let $s=10^{Cn^2}$, $k=2n+2$ and $d=10^{C'n}$ from the proof above and use Definition~\ref{def:rob-vc} and Theorem~\ref{thm:goldberg}  to get a VC dimension of the robust loss upper bound of $O(k\log(sd))=O(n^3)$.\footnote{Note that Corollary~2.4 in \cite{goldberg1995bounding} uses this reasoning.}
\end{proof}

\subsection{Robust Learning against Precision-Bounded Adversaries}
\label{sec:adv-bounded-precision}

It is possible to obtain some relatively straightforward robustness guarantees for classes with infinite Littlestone dimension if there exists a sufficiently large margin between classes (in which case the exact-in-the-ball and  constant-in-the-ball notions of robustness coincide). 
However, some of these results have already been derived in the literature.
See, e.g., \citep{cullina2018pac} for the sample complexity of halfspaces in the constant-in-the-ball realizable setting w.r.t. $\ell_p$-norm adversaries, which improves on the sample complexity bound of Theorem~\ref{thm:sc-ub-ltf-real} by being linear -- vs cubic -- in the input dimension; together with a mistake bound for Perceptron, we get $\LEQ$ bounds.\footnote{In this case, we would need  a margin between the sets $B_\rho(\supp(D_0))$ and $B_\rho(\supp(D_1))$, as these are the sets of potential counterexamples -- the condition $B_\rho(\supp(D_0))\cap B_\rho(\supp(D_1))=\emptyset$ is not sufficient in itself to get guarantees for hypotheses with infinite Littlestone dimension. See \citep{montasser2021adversarially} for both upper and lower bounds in this setting.}
We also note that a previous version of this work \citep{gourdeau2022when} contained an erroneous statement on the robust learnability of halfspaces with margin in $\R^n$. 
An erratum is included in Appendix~\ref{app:erratum}.

Instead, in this section, we look at robust learning  problems in which the decision boundary can cross the perturbation region, but where the adversary's precision is limited.
We use ideas from \cite{ben2009agnostic} concerning hypotheses with margins in the online learning framework. 
Note, however, that here the margin does not represent sufficient distance between classes or a hypothesis' confidence, but rather a region of the instance space that is too costly for the adversary to access (e.g., the number of bits needed to express an adversarial example is too large). 
A thorough comparison with online learning and the work of \cite{ben2009agnostic} appears at the end of this section.

Examining the proof that the Littlestone dimension of thresholds is infinite (see Appendix~\ref{app:online}), the key assumption is that the adversary has \emph{infinite} precision. 
This is perhaps not a reasonable assumption to make in practice. 
More precisely, in the construction of the Littlestone tree, each counterexample given requires an additional bit to be described, as the remainder of the interval $[0,1]$ is split in two at each prediction. 
Our work in this section formally and more generally addresses this potential issue.

We now define the meaning of bounding an adversary's precision in the context of robust learning, which is depicted in Figure~\ref{fig:bounded-precision}.

\begin{figure}
\begin{center}
\includegraphics[scale=0.3]{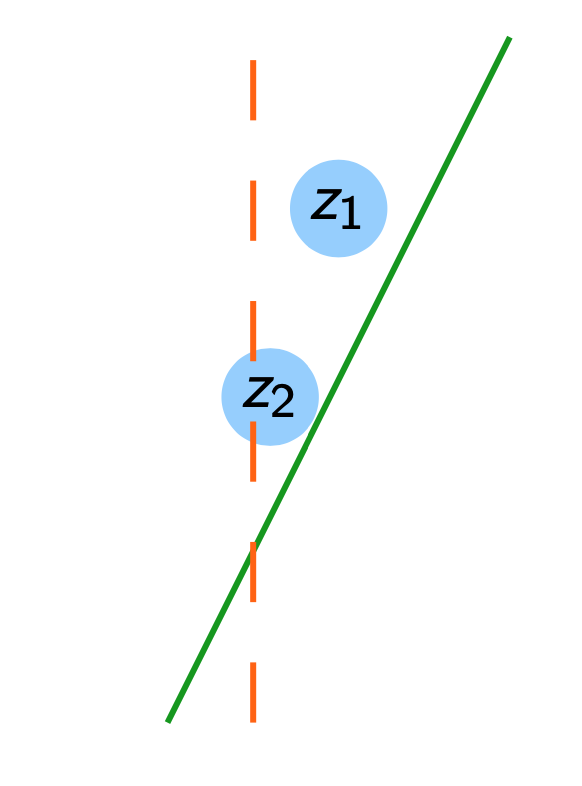}
\end{center}
\caption{The dotted line is the hypothesis $h$, and the solid line, the target $c$. The adversary has precision $\tau$. The shaded regions represent the set $B_\tau(z_i)$. The counterexample $z_1$ is valid as $c$ and $h$ disagree on all of $B_\tau(z_1)$ and both functions are constant in this region, but $z_2$ is not as $c$ and $h$ agree on part of $B_\tau(z_2)$.}
\label{fig:bounded-precision}
\end{figure}

\begin{definition}[Precision-Bounded Adversary]
Let $(\X,d)$ be a metric space, and let an adversary $\adv$ have budget $\rho$. 
We say that $\adv$ is \emph{precision-bounded} by $\tau$, if for target $c$, hypothesis $h$, and input $x$, $\adv$ can only return counterexamples $z\in B_\rho(x)$ such that $h$ and  $c$ are both constant and disagree on the whole region $B_{\tau}(z)$ and $B_{\tau}(z)\subseteq B_\rho(x)$.
\end{definition}

\begin{definition}[Littlestone Trees of Precision $\tau$]
\label{def:lit-tree-tau}
A Littlestone tree of precision $\tau$ for a hypothesis class $\mathcal{H}$ on metric space $(\X,d)$ is a complete binary tree $T$ of depth $d$ whose internal nodes are instances $x\in\X$.
Each edge is labelled with $-$ or $+$ and corresponds to the potential labels of the parent node $x$ and the region $B_\tau(x)$.
Each path from the root to a leaf must be consistent with some $h\in\mathcal{H}$, i.e. if $x_1,\dots,x_d$ with labellings $y_1,\dots,y_d$ is a path in $T$, there must exist $h\in\mathcal{H}$ such that $h\vert_{B_\tau(x_i)}=y_i$ for all $i$. 
\end{definition}

While it is possible to have a hypothesis giving different labels to points in the region $B_\tau(x)$ in the standard setting, in the above construction, one must commit to labelling the whole region $B_\tau(x)$ either positively or negatively. 

For the remainder of the text, we will identify each leaf in a Littlestone tree $T$ with a hypothesis $h\in\H$ that is consistent with the labellings along the path from the root to this leaf. 
Note that the choice of labelling $y$ of $B_\tau(x)$ of some $x\in T$ implies that, in contrast to the standard Littlestone trees, any $h\in\H$ with $h(x)=y$ that is \emph{not} constant on  $B_\tau(x)$ cannot be consistent with any path in $T$. 
The set of consistent hypotheses on $T$ thus does \emph{not} form a partition of $\H$ in our precision-bounded setting.

We can now define the following variant of the Littlestone dimension, which is analogous to the margin-based Littlestone dimension of \cite{ben2009agnostic}.

\begin{definition}[Precision-Bounded Littlestone Dimension]
The Littlestone dimension of precision $\tau$ of a hypothesis class $\mathcal{H}$ on metric space $(\X,d)$, denoted $\Lit_\tau(\mathcal{H})$, is the depth $k$ of the largest Littlestone tree with bounded precision $\tau$ for $\mathcal{H}$. If no such $k$ exists then $\Lit(\mathcal{H})=\infty$.
\end{definition}

Note that setting $\tau=0$, i.e., there are no constraints on the nodes, we recover the Littlestone tree and Littlestone dimension definitions.
As an example, let us consider the class of threshold functions, which, when $\tau=0$, have infinite Littlestone dimension.

\begin{proposition}
Let $\tau>0$. 
The class $\thresholds_B$ of threshold functions on $[0,B]$ induce Littlestone trees of precision $\tau$ of depth bounded by $log\frac{B}{\tau}-1$. Thus $\Lit_\tau(\thresholds_B)=\lfloor\log\frac{B}{\tau}-1\rfloor$.
\end{proposition}

\begin{proof}
Let $\tau>0$ be arbitrary.
Here, the optimal strategy to construct a Littlestone tree of maximal depth is to divide the interval $[0,B]$ in two equal parts at each round.\footnote{To see that this is optimal, consider the case where the interval is not split into two equal parts.
Since we are committing to the labelling of the whole region $B_\tau(x)$ at a given node $x$, the smaller interval will result in a Littlestone subtree that has smaller depth than if the subintervals were of equal length. This results in a Littlestone tree of smaller depth as Littlestone trees must be complete.}
Given $x\in[0,B]$ and $\alpha<\alpha'\in\R$, in order to have two threshold functions $h_\alpha(x)=\mathbf{1}[x\geq \alpha]$ and $h_{\alpha'}(x)=\mathbf{1}[x\geq \alpha']$ that disagree on the whole range $[x-\tau,x+\tau]$, we need both $\alpha<x-\tau$ and $\alpha'\geq x+\tau$.
Thus, at depth $d$, we have divided $[0,B]$ into $2^d$ parts we must have $2\tau\geq B2^{-d}$, implying $\Lit_\tau(\thresholds_B)=\lfloor\log\frac{B}{\tau}-1\rfloor$.
\end{proof}
 
We remark that, exactly following the proof of online learning ($\tau=0$), $\Lit_\tau(\H)$ is a lower bound on the number of mistakes of any learner against an adversary of precision $\tau$.

\begin{theorem}
Any online learning algorithm for $\C$ has mistake bound $M\geq \Lit_\tau(\C)$ against a $\tau$-precision-bounded adversary.
\end{theorem}
\begin{proof}
Let $\A$ be any online learning algorithm for $\C$.
Let $T$ be a Littlestone tree of precision $\tau$ and depth $\Lit_\tau(\C)$ for $\C$. 
An adversary can force $\A$ to make $\Lit_\tau(\C)$ mistakes by sequentially and adaptively choosing a path in $T$ in response to $\A$'s predictions.
\end{proof}
 
Now, let us consider a version of the SOA where the adversary has precision $\tau$. 
The algorithm is identical to the SOA, except for the definition of $V^{(b)}_t$, which requires that the hypotheses are constant in the region around the prediction.

\begin{algorithm}
\caption{Precision-Bounded Standard Optimal Algorithm}
\begin{algorithmic}
\Require A hypothesis class $\mathcal{\H}$
\State $V_1\gets\H$
\For {$t=1,2,\dots$}
\State Receive example $x_t$
\State $V^{(b)}_t \gets \set{h\in V_t\given h\vert_{B_\tau(x_t)}=b}$
\State $\hat{y_t}=\underset{b}{\arg\max} \;\Lit_\tau\left( V^{(b)}_t \right)$ 
\State Receive true label $y_t$
\State $V_{t+1} \gets V^{(y_t)}_t$
\EndFor
\end{algorithmic}
\label{alg:soa-precision}
\end{algorithm}

Below, we show that this slight modification of the SOA is also optimal for cases in which the adversary is constrained by $\tau$.
This is analogous to Theorem~21 in \citep{ben2009agnostic}, who did not include their proof of optimality for brevity. 
It is included here for completeness.

\begin{theorem}
\label{thm:soa-tau-opt}
The precision-bounded Standard Optimal Algorithm makes at most $\Lit_\tau(\C)$ mistakes in the mistake-bound model of online learning when the adversary has precision $\tau$.
\end{theorem}

The proof of Theorem~\ref{thm:soa-tau-opt} will use the following result, showing that no node in the tree has a $\tau$-expansion that overlaps with the $\tau$-expansion of any of its ancestors.

\begin{proposition}
\label{prop:ancestor-intersection-empty}
Let $T$ be a Littlestone tree of precision $\tau$.
Then for any node $x\in T$ and ancestor $x'\in T$ of $x$, $B_\tau(x)\cap B_\tau(x')=\emptyset$.
\end{proposition}
\begin{proof}
Take two paths from the root to two distinct leaves, $h_0$ and $h_1$, respectively.
Let the paths branch off at $x\in T$, with $h_y$ giving label $y$ to the whole region $B_\tau(x)$.
Let $x'$ be an ancestor of $x$ in $T$, and note that $h_0=h_1=b$ on $B_\tau(x')$ for some $b\in\{0,1\}$.
Then, since $h_0$ and $h_1$ must disagree on all of $B_\tau(x)$, it follows that $B_\tau(x)\cap B_\tau(x')=\emptyset$.
\end{proof}

We are now ready to prove Theorem~\ref{thm:soa-tau-opt}.

\begin{proof}[Proof of Theorem~\ref{thm:soa-tau-opt}]
We will show that, at every mistake, the precision-bounded Littlestone dimension of the subclass $V_t$ decreases by at least 1 after receiving the true label $y_t$.

WLOG, assume that there are not $t'<t$ such that $x_{t'}\in B_\tau(x_t)$, as otherwise this implies that $V_t^{(y_{t'})}=V_t$ and $V_t^{(\neg y_{t'})}=\emptyset$, and we cannot make a mistake (note in particular that we cannot have two differently labelled points in $B_\tau(x_t)$ as otherwise this would not be a valid example for the adversary to give).

Suppose that, at time $t$, $y_t= \arg\min_b \Lit_\tau(V^{(b)}_t)$. 
Note that $V_{t+1} = V^{(y_t)}_t$.
Now, consider any two Littlestone trees  $T_{y_t}$ and  $T_{\hat{y}_t}$ of precision $\tau$ and maximal depths for $V^{(y_t)}_t$ and $V^{(\hat{y})}_t$, respectively. 
By Proposition~\ref{prop:ancestor-intersection-empty} and definition of $V^{(b)}_t$, neither tree can contain nodes whose $\tau$-expansions intersect with $B_\tau(x_t)$. 
Moreover, all hypotheses in $V^{(y_t)}_t$ and $V^{(\hat{y})}_t$ are constant on $B_\tau(x_t)$.
Hence it is possible to construct a $\tau$-constrained Littlestone tree $T$ for $V_t$ of depth $\min_b \Lit_\tau(V^{(b)}_t)+1$ (recall that $T$ must be complete).
Then $\Lit_\tau(V_t)\geq\Lit_\tau(V^{({y_t})}_t)+1=\Lit_\tau(V_{t+1})+1 $, as required.\footnote{Note that the Littlestone dimension does not necessarily decrease when $y_t=\hat{y}_t$, as we could have $V_t=V_t^{(y_t)}$.}
\end{proof}

\begin{remark}
When considering threshold functions on $[0,1]$, and given example $x_t$ to predict, the SOA's strategy is effectively to look at the labelled points in the history $x_1,\dots,x_{t-1}$ and consider the largest $x^{(0)}\in[0,1]$ with negative label and the smallest $x^{(1)}\in[0,1]$ with positive label, and predicting $y_t=\underset{b}{\arg\min} \abs{x_t - x^{(b)}}$.
\end{remark}

We now turn our attention to the robust learning of halfspaces in $(\R^n,d)$ against adversaries of precision $\tau$, where $d$ is the metric induced by the $\ell_2$ norm.
 As pointed out by \cite{ben2009agnostic}, we essentially have the same argument as the Perceptron algorithm, because, once the hypothesis is sufficiently close to the target, the adversary cannot return counterexamples near the boundary.
Note that this result can be generalized to $\ell_p$ norms.
 Figure~\ref{fig:ltf-bounded-precision} depicts the argument of the proof of Theorem~\ref{thm:ltf-real-precision}.
 
\begin{theorem}
\label{thm:ltf-real-precision}
Fix constants $B,\tau>0$.
Let the adversary's budget $\rho$ be measured by the $\ell_2$ norm.
Let $\ltfreal$ be the class of halfspaces on $\R^n$ where the instance space is restricted to points $x\in\R^n$ with $\norm{x}_2\leq B-\rho$. 
Then, $\ltfreal$  is distribution-free $\rho$-robustly learnable against an adversary of precision $\tau$ using the $\EX$ and $\rho$-$\LEQ$ oracles with sample complexity $m(n,\epsilon,\delta)=O(\frac{1}{\epsilon}( n^3 + \log (1/\delta)))$ and query complexity $r(n,\epsilon,\delta)=m(n,\epsilon,\delta)\cdot\frac{B^2}{\tau^2}$.
Note that this is query-efficient if $\frac{B^2}{\tau^2}=\poly(n)$.
\end{theorem}

\begin{proof}
The sample complexity bound follows from Theorem~\ref{thm:sc-ub-ltf-real}.
The query upper bound  follows from Lemma~\ref{lemma:rob-mistake-bound}  and the mistake bound for the Perceptron algorithm (see Theorem~\ref{thm:mistake-bound-perceptron}).
To see that the bound for Perceptron can be used, note that the adversary having precision $\tau$ implies that any consistent target function $c(x)=a^\top x +a_0$ and any counterexample $z$ will satisfy the conditions 
 (i) $\norm{z}_2\leq B$ and (ii) $\tau \leq \frac{c(z)(a^\top z)}{\norm{z}_2}$ from Theorem~\ref{thm:mistake-bound-perceptron}.
\end{proof}

Note that the dependence on $\tau$ in the mistake bound, and thus the $\LEQ$ upper bound, is $1/\tau^2$, in contrast to the dependence of $\log 1/\tau$ for thresholds.

\begin{remark}
The precision bounds on an adversary allow us to achieve robust learning guarantees where they were previously impossible to obtain.
However, since we expect the target and hypothesis to be constant in balls of radius $\tau$ to get counterexamples, the set-up has some limitations.
Indeed, if this assumption does not hold, robust learning can become trivial.
For instance, for parities on $\boolhc$ (which are highly unstable as a single bit flip can cause a label change) and for any $\tau\geq1$, we have that any $\tau$-bounded-precision adversary will not be able to return any counterexamples, implying that we are operating the standard PAC setting. 
\end{remark}

\begin{figure}
\begin{center}
\includegraphics[scale=0.24]{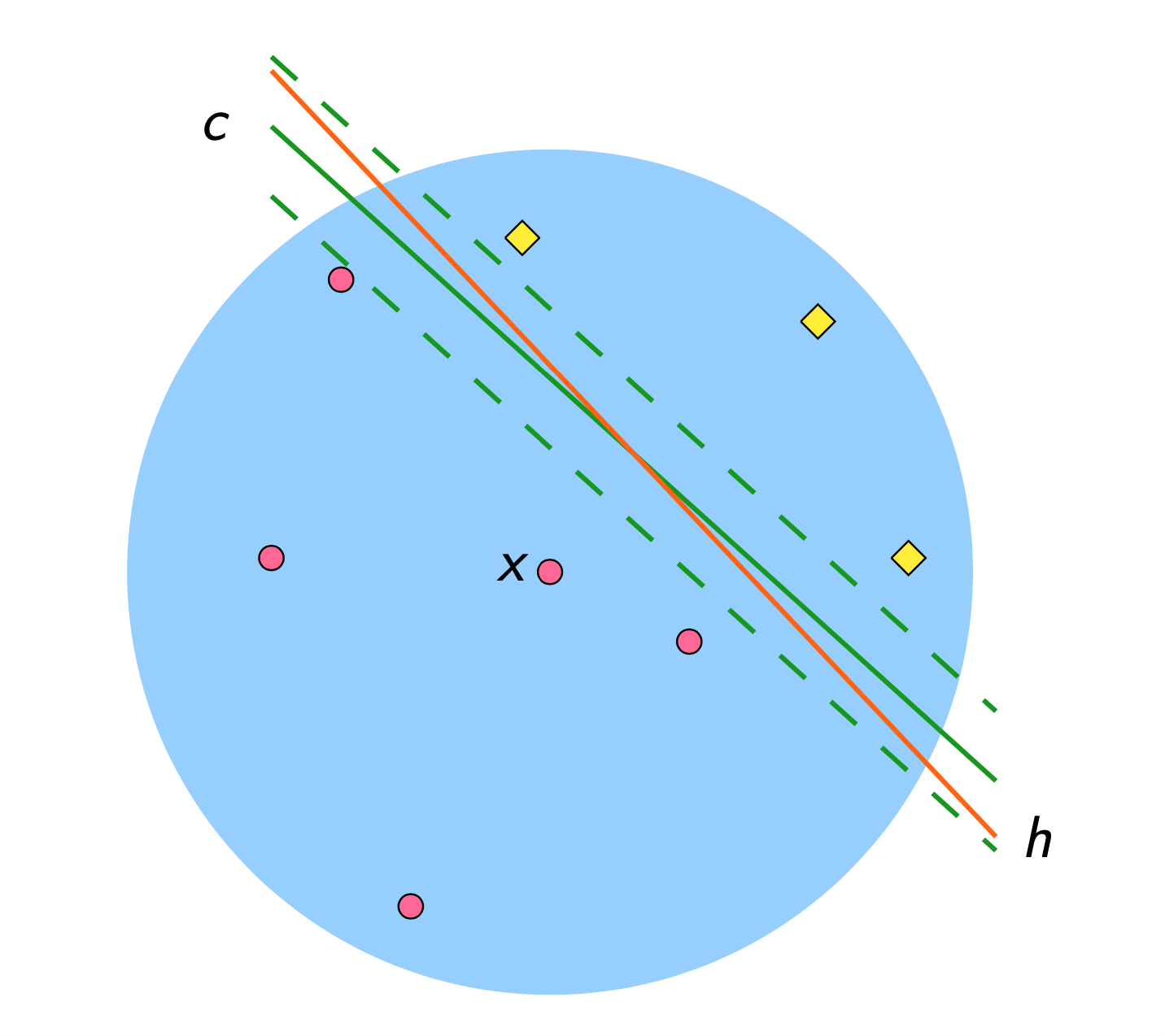}
\end{center}
\caption{A visual representation of the proof of Theorem~\ref{thm:ltf-real-precision}. The dotted lines on either side of the target $c$ represent a margin of $\tau/2$. Any hypothesis within the dotted lines in the (shaded) perturbation region ensures that an adversary of bounded precision $\tau$ cannot return any counterexamples. Finally, counterexamples must be labelled according to the target $c$, and both $h$ and $c$ are not constant on $B_\rho(x)$.}
\label{fig:ltf-bounded-precision}
\end{figure}

\paragraph*{Comparison with online learning.}
Note that if we set $\rho$ to be large enough so that the perturbation region is the whole instance space $\X$ for any point $x$, we (almost) recover the adversary model in the online learning setting.
The only distinction is that, in online learning, the learner is given a point $x_t$ at time $t$ to classify (implicitly classifying the whole region $B_\tau(x_t)$), rather than committing to a hypothesis $h$ on the whole instance space. 
The adversary (or ``nature'', if a target must be chosen a priori) reveals the true label after a prediction is made. 

In the mistake-bound model, the only constraint is that there exists a concept in $\H$ that is consistent with the labelled sequence $(x_1,y_1,),\dots,(x_t,y_t)$ seen so far.
When working with a precision-bounded adversary, we  are implicitly asking the adversary to not give counterexamples too close to the boundary.
Then, in the mistake-bound model, this translates into the adversary giving a point $x_t$ to predict such that there does not exist time steps $t',t''$ where $y_{t'}\neq y_{t''}$ and $B_\tau(x_{t'})$ and $B_\tau(x_{t''})$ both intersect with $B_\tau(x_t)$, hence a \emph{margin}. 
As previously mentioned, margin-based complexity measures for online learning adapted from \cite{ben2009agnostic} have been used in this section. 

A subtle distinction between the definitions we gave above and that of \cite{ben2009agnostic} is that the latter defined margin-based Littlestone trees and Littlestone dimension for margin-based \emph{hypothesis classes}. 
They require that the hypothesis class $\H$ satisfies the following: for all $h\in\H$, $h$ is of the form $\X\rightarrow\R$, and the prediction rule is 
\begin{equation}
\phi(h(x))=\frac{\sgn(h(x))+1}{2}
\enspace,
\end{equation}
where the magnitude $\abs{h(x)}$ is the \emph{confidence} in the prediction. 
The $\mu$-\emph{margin-mistake} on an example $(x,y)$ is defined as
\begin{equation}
\label{eqn:margin-mistake}
\abs{h(x)-y}_\mu=
\begin{cases}
0 & \text{if }\phi(h(x))=y \wedge \abs{h(x)}\geq \mu \\
1 & \text{otherwise}
\end{cases}
\enspace.
\end{equation}

For us, since it is the \emph{adversary} that is restricted in its precision, we instead consider any hypothesis class where the concepts are boolean functions whose domain is a metric space $(\X,d)$. 
Rather than having the condition $\abs{h(x)}\geq \mu$ from Equation~\ref{eqn:margin-mistake}, we encode a margin representing the precision $\tau$ by the requirement that hypotheses must be constant in the $\tau$-expansion around any point in the Littlestone trees.
This difference is not only stylistic, but also concerns the semantics of the margin. 
Our definition moreover implies a uniform margin on the instance space, while one from \cite{ben2009agnostic} can fluctuate in the instance space based on the classifier's confidence. 
However, the tools and techniques used here don't differ much in essence from the ones in \cite{ben2009agnostic}. 
The main novelty is the meaning of the notion of margin and its study in the context of robust learning.

\section{A Local Membership Query Lower Bound for Conjunctions}
\label{sec:lmq-lb}

In this section, we show that the amount of data needed to $\rho$-robustly learn conjunctions under the uniform distribution has an exponential dependence on the adversary's budget $\rho$ when the learner only has access to the $\EX$ and $\LMQ$ oracles.
Here, the lower bound on the sample drawn from the example oracle is $2^\rho$, which is the same as the lower bound for \emph{monotone} conjunctions derived in \cite{gourdeau2022sample}, and the local membership query lower bound is $2^{\rho-1}$. 
The result relies on showing there there exists a family of conjunctions that remain indistinguishable from each other on any sample of size $2^\rho$ and any sequence of $2^{\rho-1}$ LMQs with constant probability.

\begin{theorem}
\label{thm:conj-lmq-lb}
Fix a monotone increasing function $\rho:\N\rightarrow \N$ satisfying $2 \leq \rho(n) \leq n/4$ for all $n$.
Then, for any query radius $\lambda$, any $\rho(n)$-robust learning algorithm for the class $\Conj$ with access to the $\EX$ and $\lambda$-$\LMQ$ oracles has sample and query complexity lower bounds of $2^\rho$ and $2^{\rho-1}$ under the uniform distribution.
\end{theorem}

\begin{proof}
Let $D$ be the uniform distribution and without loss of generality let $\rho \geq 2$.
Fix two disjoint sets $I_1$ and $I_2$ of $2\rho$ indices in $[n]$, which will be the set of variables appearing in potential target conjunctions $c_1$ and $c_2$, respectively (i.e., their support).
We have $2^{4\rho}$ possible pairs of such conjunctions, as each variable can appear as a positive or negative literal.

Let us consider a randomly drawn sample $S$ of size $2^\rho$.
We will first consider what happens when all the examples in $S$ and the queried inputs $S'$ are negatively labelled.
Each negative example $x\in S$ allows us to remove at most $2^{2\rho+1}$ pairs from the possible set of pairs of conjunctions, as each component $x_{I_1}$ and $x_{I_2}$ removes at most one conjunction from the possible targets. 
By the same reasoning, each LMQ that returns a negative example can remove at most $2^{2\rho+1}$ pairs of conjunctions.
Note that the parameter $\lambda$ is irrelevant in this setting as each LMQ can only test one concept pair.
Thus, after seeing any random sample of size $2^\rho$ and querying any $2^{\rho -1}$ points, there remains 
\begin{equation}
\label{eqn:consistent-lmq}
\frac{2^{4\rho} - 2^{3\rho+1}-2^{3\rho}}{2^{4\rho}}\geq 1/4
\end{equation}
of the initial conjunction pairs that label all points in $S$ and $S'$ negatively. 
Then, choosing a pair $(c_1,c_2)$ of possible target conjunctions uniformly at random and then choosing $c$ uniformly at random gives at least a $1/4$ chance that $S$ and $S'$ only contain negative examples (both conjunctions are consistent with this).

Moreover, note that any  two conjunctions in a pair will have a robust risk lower bounded by $15/32$ against each other under the uniform distribution (see Lemma~\ref{lemma:rob-loss-mon-conj} in  Appendix~\ref{app:useful}).
Thus, any learning algorithm $\A$ with LMQ query budget $m'=2^{\rho-1}$ and strategy $\sigma:(\boolhc \times \set{0,1})^m\rightarrow(\boolhc \times \set{0,1})^{m'}$ (note that the queries can be adaptive) can do no better than to guess which of $c_1$ or $c_2$ is the target if they are both consistent on the augmented sample $S\cup\sigma(S)$, giving an expected robust risk lower bounded by a constant.
Letting $\E$ be the event that all points in both $S$ and $\sigma(S)$ are labelled zero, we get
\begin{align*}
\eval{c,S}{\risk_\rho^D(\A(S\cup\sigma(S)),c)}
&\geq \Prob{c,S}{\E}\eval{c,S}{\risk_\rho^D(\A(S\cup\sigma(S)),c)\given \E} \tag{Law of Total Expectation}
 \\
&\geq \frac{1}{4}\;\eval{c,S}{\risk_\rho^D(\A(S\cup\sigma(S)),c)\given \E}
\tag{Equation~\ref{eqn:consistent-lmq}} \\
&= \frac{1}{4}\cdot\frac{1}{2}\;\eval{S}{\risk_\rho^D(\A(S\cup\sigma(S)),c_1)+\risk_\rho^D(\A(S\cup\sigma(S)),c_2)\given \E} 
\tag{Random choice of $c$} \\
&\geq\frac{1}{8}\;\eval{S}{\risk_\rho^D(c_1,c_2)\given \E} \tag{Lemma~\ref{lemma:rob-triangle}}\\
&>\frac{1}{8}\cdot\frac{15}{32} \tag{Lemma~\ref{lemma:rob-loss-mon-conj}}\\
&=\frac{15}{256}
\enspace,
\end{align*}
which completes the proof.
\end{proof}

We use the term \emph{robustness threshold} from \cite{gourdeau2021hardness} to denote an adversarial budget function $\rho:\N\rightarrow\R$ of the input dimension $n$ such that, if the adversary is allowed perturbations of magnitude $\rho(n)$, then there exists a sample-efficient $\rho(n)$-robust learning algorithm, and if the adversary's budget is $\omega(\rho(n))$, then there does not exist such an algorithm.
Robustness thresholds are distribution-dependent when the learner only has access to the example oracle $\EX$, as seen in \citep{gourdeau2021hardness,gourdeau2022sample}.
Now, since the local membership query lower bound  above has an exponential dependence on $\rho$, any perturbation budget $\omega(\log n)$ will require a sample and query complexity that is superpolynomial in $n$, giving the following corollary.

\begin{corollary}
The robustness threshold of the class $\Conj$ under the uniform distribution with access to $\EX$ and an $\LMQ$ oracle is $\Theta(\log(n))$.
\end{corollary}

The robustness threshold above is the same as when only using the $\EX$ oracle \citep{gourdeau2021hardness}.
Finally, since decision lists and halfspaces both subsume conjunctions, the lower bound of Theorem~\ref{thm:conj-lmq-lb} also holds for these classes.

\section{Conclusion}
\label{sec:conclusion}

We have shown that local membership queries do not change the robustness threshold of conjunctions, or any superclass, under the uniform distribution. 
However, access to a $\rho$-local \emph{equivalence} query oracle allows us to develop robust ERM algorithms.
We have introduced  the notion of VC dimension of the robust loss to determine sample complexity bounds and have used online learning results to derive query complexity bounds. 
We have moreover adapted the PAC learning algorithm for conjunctions for this setting and have greatly improved its query complexity compared to the general case. 
Finally, we have studied halfspaces, both in the boolean hypercube and continuous settings. 
The latter is, to our knowledge, the first robust learning algorithm with respect to the exact-in-the-ball notion of robustness for a non-trivial concept class in $\R^n$.
Overall, we have shown that the $\LEQ$ oracle (or a similar type of oracle) is \emph{essential} to ensure the \emph{distribution-free} robust learning of commonly studied concept classes in our setting. 
Note that this is in contrast with standard PAC learning with the $\EX$ and $\EQ$ oracles, where EQs don't give more power to learner. 

 We finally outline various avenues for future research:
\begin{enumerate}
\item Can we give a more fine-grained picture of the sample and query complexity tradeoff outlined in Remark~\ref{rmk:rho-tradeoff},  e.g., by improving $\LEQ$ query upper bounds when $\rho$ is small? 
\item Can we derive sample and query lower bounds for robust learning with an $\LEQ$ oracle? 
\item The $\LMQ$ lower bound from Section~\ref{sec:lmq-lb} was derived for conjunctions. The technique does not work for monotone conjunctions.\footnote{For a given set of indices $I$, there exists only one monotone conjunction using all indices in $I$.} Can we get a similar LMQ lower bound where the dependence on $\rho$ is exponential for monotone conjunctions, or it is possible to robustly learn them with $o(2^\rho)$ local membership queries? 
\end{enumerate}

\section*{Acknowledgements}

MK and PG received funding from the ERC under the European Union’s Horizon 2020 research and innovation programme (FUN2MODEL, grant agreement No.~834115).

\bibliographystyle{apalike}
\bibliography{references}

\begin{thebibliography}{}

\bibitem[Angluin, 1987]{angluin1987learning}
Angluin, D. (1987).
\newblock Learning regular sets from queries and counterexamples.
\newblock {\em Information and computation}, 75(2):87--106.

\bibitem[Angluin, 1988]{angluin1988queries}
Angluin, D. (1988).
\newblock Queries and concept learning.
\newblock {\em Machine learning}, 2(4):319--342.

\bibitem[Angluin, 1990]{angluin1990negative}
Angluin, D. (1990).
\newblock Negative results for equivalence queries.
\newblock {\em Machine Learning}, 5(2):121--150.

\bibitem[Angluin and Kharitonov, 1995]{angluin1995when}
Angluin, D. and Kharitonov, M. (1995).
\newblock When won't membership queries help?
\newblock {\em Journal of Computer and System Sciences}, 50(2):336--355.

\bibitem[Ashtiani et~al., 2020]{ashtiani2020black}
Ashtiani, H., Pathak, V., and Urner, R. (2020).
\newblock Black-box certification and learning under adversarial perturbations.
\newblock In {\em International Conference on Machine Learning}, pages
  388--398. PMLR.

\bibitem[Awasthi et~al., 2013]{awasthi2013learning}
Awasthi, P., Feldman, V., and Kanade, V. (2013).
\newblock Learning using local membership queries.
\newblock In {\em Conference on Learning Theory}, volume~30, pages 1--34.

\bibitem[Awasthi et~al., 2020]{awasthi2020adversarial}
Awasthi, P., Frank, N., and Mohri, M. (2020).
\newblock Adversarial learning guarantees for linear hypotheses and neural
  networks.
\newblock In {\em International Conference on Machine Learning}, pages
  431--441. PMLR.

\bibitem[Bary-Weisberg et~al., 2020]{bary2020distribution}
Bary-Weisberg, G., Daniely, A., and Shalev-Shwartz, S. (2020).
\newblock Distribution free learning with local queries.
\newblock In {\em Algorithmic Learning Theory}, pages 133--147. PMLR.

\bibitem[Baum and Lang, 1992]{baum1992query}
Baum, E.~B. and Lang, K. (1992).
\newblock Query learning can work poorly when a human oracle is used.
\newblock In {\em International joint conference on neural networks}, volume~8,
  page~8. Beijing China.

\bibitem[Ben-David et~al., 2009]{ben2009agnostic}
Ben-David, S., P{\'a}l, D., and Shalev-Shwartz, S. (2009).
\newblock Agnostic online learning.
\newblock In {\em Conference on Learning Theory}, volume~3, page~1.

\bibitem[Bhattacharjee et~al., 2021]{bhattacharjee2021sample}
Bhattacharjee, R., Jha, S., and Chaudhuri, K. (2021).
\newblock Sample complexity of robust linear classification on separated data.
\newblock In {\em International Conference on Machine Learning}, pages
  884--893. PMLR.

\bibitem[Biggio et~al., 2013]{biggio2013evasion}
Biggio, B., Corona, I., Maiorca, D., Nelson, B., {\v{S}}rndi{\'c}, N., Laskov,
  P., Giacinto, G., and Roli, F. (2013).
\newblock Evasion attacks against machine learning at test time.
\newblock In {\em Joint European conference on machine learning and knowledge
  discovery in databases}, pages 387--402. Springer.

\bibitem[Biggio and Roli, 2018]{biggio2017wild}
Biggio, B. and Roli, F. (2018).
\newblock Wild patterns: Ten years after the rise of adversarial machine
  learning.
\newblock In {\em Proceedings of the 2018 ACM SIGSAC Conference on Computer and
  Communications Security}, pages 2154--2156.

\bibitem[Bshouty, 1993]{bshouty1993exact}
Bshouty, N.~H. (1993).
\newblock Exact learning via the monotone theory.
\newblock In {\em Proceedings of 1993 IEEE 34th Annual Foundations of Computer
  Science}, pages 302--311. IEEE.

\bibitem[Bubeck et~al., 2019]{bubeck2019adversarial}
Bubeck, S., Lee, Y.~T., Price, E., and Razenshteyn, I. (2019).
\newblock Adversarial examples from computational constraints.
\newblock In {\em Proceedings of the 36th International Conference on Machine
  Learning}, volume~97 of {\em Proceedings of Machine Learning Research}, pages
  831--840, Long Beach, California, USA. PMLR.

\bibitem[Camacho and McIlraith, 2019]{camacho2019learning}
Camacho, A. and McIlraith, S.~A. (2019).
\newblock Learning interpretable models expressed in linear temporal logic.
\newblock In {\em Proceedings of the International Conference on Automated
  Planning and Scheduling}, volume~29, pages 621--630.

\bibitem[Cullina et~al., 2018]{cullina2018pac}
Cullina, D., Bhagoji, A.~N., and Mittal, P. (2018).
\newblock {PAC}-learning in the presence of evasion adversaries.
\newblock {\em Advances in Neural Information Processing Systems}.

\bibitem[Dalvi et~al., 2004]{dalvi2004adversarial}
Dalvi, N., Domingos, P., Sanghai, S., Verma, D., et~al. (2004).
\newblock Adversarial classification.
\newblock In {\em Proceedings of the tenth ACM SIGKDD international conference
  on Knowledge discovery and data mining}, pages 99--108. ACM.

\bibitem[Diakonikolas et~al., 2020]{diakonikolas2020complexity}
Diakonikolas, I., Kane, D.~M., and Manurangsi, P. (2020).
\newblock The complexity of adversarially robust proper learning of halfspaces
  with agnostic noise.
\newblock {\em Advances in Neural Information Processing Systems},
  33:20449--20461.

\bibitem[Diochnos et~al., 2018]{diochnos2018adversarial}
Diochnos, D., Mahloujifar, S., and Mahmoody, M. (2018).
\newblock Adversarial risk and robustness: General definitions and implications
  for the uniform distribution.
\newblock In {\em Advances in Neural Information Processing Systems}.

\bibitem[Diochnos et~al., 2020]{diochnos2020lower}
Diochnos, D.~I., Mahloujifar, S., and Mahmoody, M. (2020).
\newblock Lower bounds for adversarially robust {PAC} learning under evasion
  and hybrid attacks.
\newblock In {\em 2020 19th IEEE International Conference on Machine Learning
  and Applications (ICMLA)}, pages 717--722.

\bibitem[Dreossi et~al., 2019]{dreossi2019formalization}
Dreossi, T., Ghosh, S., Sangiovanni-Vincentelli, A., and Seshia, S.~A. (2019).
\newblock A formalization of robustness for deep neural networks.
\newblock {\em arXiv preprint arXiv:1903.10033}.

\bibitem[Fawzi et~al., 2018a]{fawzi2018adversarial}
Fawzi, A., Fawzi, H., and Fawzi, O. (2018a).
\newblock Adversarial vulnerability for any classifier.
\newblock {\em Advances in neural information processing systems}, 31.

\bibitem[Fawzi et~al., 2018b]{fawzi2018analysis}
Fawzi, A., Fawzi, O., and Frossard, P. (2018b).
\newblock Analysis of classifiers robustness to adversarial perturbations.
\newblock {\em Machine Learning}, 107(3):481--508.

\bibitem[Fawzi et~al., 2016]{fawzi2016robustness}
Fawzi, A., Moosavi-Dezfooli, S.-M., and Frossard, P. (2016).
\newblock Robustness of classifiers: from adversarial to random noise.
\newblock In {\em Advances in Neural Information Processing Systems}, pages
  1632--1640.

\bibitem[Feige et~al., 2015]{feige2015learning}
Feige, U., Mansour, Y., and Schapire, R. (2015).
\newblock Learning and inference in the presence of corrupted inputs.
\newblock In {\em Conference on Learning Theory}, pages 637--657.

\bibitem[Garg et~al., 2020]{garg2020adversarially}
Garg, S., Jha, S., Mahloujifar, S., and Mohammad, M. (2020).
\newblock Adversarially robust learning could leverage computational hardness.
\newblock In {\em Algorithmic Learning Theory}, pages 364--385. PMLR.

\bibitem[Gilmer et~al., 2018]{gilmer2018adversarial}
Gilmer, J., Metz, L., Faghri, F., Schoenholz, S.~S., Raghu, M., Wattenberg, M.,
  and Goodfellow, I. (2018).
\newblock Adversarial spheres.
\newblock {\em arXiv preprint arXiv:1801.02774}.

\bibitem[Goldberg and Jerrum, 1995]{goldberg1995bounding}
Goldberg, P.~W. and Jerrum, M.~R. (1995).
\newblock Bounding the vapnik-chervonenkis dimension of concept classes
  parameterized by real numbers.
\newblock {\em Machine Learning}, 18(2-3):131--148.

\bibitem[Gourdeau et~al., 2019]{gourdeau2019hardness}
Gourdeau, P., Kanade, V., Kwiatkowska, M., and Worrell, J. (2019).
\newblock On the hardness of robust classification.
\newblock In {\em Advances in Neural Information Processing Systems}, pages
  7444--7453.

\bibitem[Gourdeau et~al., 2021]{gourdeau2021hardness}
Gourdeau, P., Kanade, V., Kwiatkowska, M., and Worrell, J. (2021).
\newblock On the hardness of robust classification.
\newblock {\em Journal of Machine Learning Research}, 22.

\bibitem[Gourdeau et~al., 2022a]{gourdeau2022sample}
Gourdeau, P., Kanade, V., Kwiatkowska, M., and Worrell, J. (2022a).
\newblock Sample complexity bounds for robustly learning decision lists against
  evasion attacks.
\newblock In {\em International Joint Conference in Artificial Intelligence}.

\bibitem[Gourdeau et~al., 2022b]{gourdeau2022when}
Gourdeau, P., Kanade, V., Kwiatkowska, M., and Worrell, J. (2022b).
\newblock When are local queries useful?
\newblock In {\em Advances in Neural Information Processing Systems}.

\bibitem[Jackson, 1997]{jackson1997efficient}
Jackson, J.~C. (1997).
\newblock An efficient membership-query algorithm for learning dnf with respect
  to the uniform distribution.
\newblock {\em Journal of Computer and System Sciences}, 55(3):414--440.

\bibitem[Khim et~al., 2019]{khim2019adversarial}
Khim, J., Jog, V., and Loh, P.-L. (2019).
\newblock Adversarial influence maximization.
\newblock In {\em 2019 IEEE International Symposium on Information Theory
  (ISIT)}, pages 1--5. IEEE.

\bibitem[Littlestone, 1988]{littlestone1988learning}
Littlestone, N. (1988).
\newblock Learning quickly when irrelevant attributes abound: A new
  linear-threshold algorithm.
\newblock {\em Machine learning}, 2(4):285--318.

\bibitem[Lowd and Meek, 2005a]{lowd2005adversarial}
Lowd, D. and Meek, C. (2005a).
\newblock Adversarial learning.
\newblock In {\em Proceedings of the eleventh ACM SIGKDD international
  conference on Knowledge discovery in data mining}, pages 641--647. ACM.

\bibitem[Lowd and Meek, 2005b]{lowd2005good}
Lowd, D. and Meek, C. (2005b).
\newblock Good word attacks on statistical spam filters.
\newblock In {\em Fifth Conference on Email and Anti-Spam (CEAS)}, volume 2005.

\bibitem[Mahloujifar et~al., 2019]{mahloujifar2019curse}
Mahloujifar, S., Diochnos, D.~I., and Mahmoody, M. (2019).
\newblock The curse of concentration in robust learning: Evasion and poisoning
  attacks from concentration of measure.
\newblock {\em AAAI Conference on Artificial Intelligence}.

\bibitem[Mahloujifar and Mahmoody, 2019]{mahloujifar2019can}
Mahloujifar, S. and Mahmoody, M. (2019).
\newblock Can adversarially robust learning leveragecomputational hardness?
\newblock In {\em Algorithmic Learning Theory}, pages 581--609. PMLR.

\bibitem[Mohri et~al., 2012]{mohri2012foundations}
Mohri, M., Rostamizadeh, A., and Talwalkar, A. (2012).
\newblock {\em Foundations of machine learning}.
\newblock MIT press.

\bibitem[Montasser et~al., 2019]{montasser2019vc}
Montasser, O., Hanneke, S., and Srebro, N. (2019).
\newblock {VC} classes are adversarially robustly learnable, but only
  improperly.
\newblock In {\em Conference on Learning Theory}, pages 2512--2530. PMLR.

\bibitem[Montasser et~al., 2020]{montasser2020reducing}
Montasser, O., Hanneke, S., and Srebro, N. (2020).
\newblock Reducing adversarially robust learning to non-robust pac learning.
\newblock {\em Advances in Neural Information Processing Systems},
  33:14626--14637.

\bibitem[Montasser et~al., 2021]{montasser2021adversarially}
Montasser, O., Hanneke, S., and Srebro, N. (2021).
\newblock Adversarially robust learning with unknown perturbation sets.
\newblock In {\em Conference on Learning Theory}, pages 3452--3482. PMLR.

\bibitem[Okudono et~al., 2020]{okudono2020weighted}
Okudono, T., Waga, M., Sekiyama, T., and Hasuo, I. (2020).
\newblock Weighted automata extraction from recurrent neural networks via
  regression on state spaces.
\newblock In {\em Proceedings of the AAAI Conference on Artificial
  Intelligence}, volume~34, pages 5306--5314.

\bibitem[Pydi and Jog, 2021]{pydi2021many}
Pydi, M.~S. and Jog, V. (2021).
\newblock The many faces of adversarial risk.
\newblock {\em Advances in Neural Information Processing Systems}, 34.

\bibitem[Renegar, 1992]{renegar1992computational}
Renegar, J. (1992).
\newblock On the computational complexity and geometry of the first-order
  theory of the reals. part i: Introduction. preliminaries. the geometry of
  semi-algebraic sets. the decision problem for the existential theory of the
  reals.
\newblock {\em Journal of symbolic computation}, 13(3):255--299.

\bibitem[Sauer, 1972]{sauer1972density}
Sauer, N. (1972).
\newblock On the density of families of sets.
\newblock {\em Journal of Combinatorial Theory, Series A}, 13(1):145--147.

\bibitem[Shafahi et~al., 2019]{shafahi2018adversarial}
Shafahi, A., Huang, W.~R., Studer, C., Feizi, S., and Goldstein, T. (2019).
\newblock Are adversarial examples inevitable?
\newblock In {\em 7th International Conference on Learning Representations
  (ICLR 2019)}.

\bibitem[Shelah, 1972]{shelah1972combinatorial}
Shelah, S. (1972).
\newblock A combinatorial problem; stability and order for models and theories
  in infinitary languages.
\newblock {\em Pacific Journal of Mathematics}, 41(1):247--261.

\bibitem[Shih et~al., 2019]{shih2019verifying}
Shih, A., Darwiche, A., and Choi, A. (2019).
\newblock Verifying binarized neural networks by angluin-style learning.
\newblock In {\em International Conference on Theory and Applications of
  Satisfiability Testing}, pages 354--370. Springer.

\bibitem[Szegedy et~al., 2013]{szegedy2013intriguing}
Szegedy, C., Zaremba, W., Sutskever, I., Bruna, J., Erhan, D., Goodfellow, I.,
  and Fergus, R. (2013).
\newblock Intriguing properties of neural networks.
\newblock In {\em International Conference on Learning Representations}.

\bibitem[Tsipras et~al., 2019]{tsipras2019robustness}
Tsipras, D., Santurkar, S., Engstrom, L., Turner, A., and Madry, A. (2019).
\newblock Robustness may be at odds with accuracy.
\newblock In {\em International Conference on Learning Representations}.

\bibitem[Valiant, 1984]{valiant1984theory}
Valiant, L.~G. (1984).
\newblock A theory of the learnable.
\newblock In {\em Proceedings of the sixteenth annual ACM Symposium on Theory
  of computing}, pages 436--445. ACM.

\bibitem[Viallard et~al., 2021]{viallard2021pac}
Viallard, P., VIDOT, E.~G., Habrard, A., and Morvant, E. (2021).
\newblock A {PAC-B}ayes analysis of adversarial robustness.
\newblock {\em Advances in Neural Information Processing Systems}, 34.

\bibitem[Weiss et~al., 2018]{weiss2018extracting}
Weiss, G., Goldberg, Y., and Yahav, E. (2018).
\newblock Extracting automata from recurrent neural networks using queries and
  counterexamples.
\newblock In {\em International Conference on Machine Learning}, pages
  5247--5256. PMLR.

\bibitem[Weiss et~al., 2019]{weiss2019learning}
Weiss, G., Goldberg, Y., and Yahav, E. (2019).
\newblock Learning deterministic weighted automata with queries and
  counterexamples.
\newblock {\em Advances in Neural Information Processing Systems}, 32.

\bibitem[Yin et~al., 2019]{yin2019rademacher}
Yin, D., Kannan, R., and Bartlett, P. (2019).
\newblock Rademacher complexity for adversarially robust generalization.
\newblock In {\em International conference on machine learning}, pages
  7085--7094. PMLR.

\end{thebibliography}




\appendix

\section{Preliminaries}

\subsection{The PAC Framework}
\label{app:pac}

\begin{definition}[PAC Learning, \cite{valiant1984theory}]
Let $\C_n$ be a concept class over $\X_n$ and let $\C=\bigcup_{n\in\N}\C_n$.
We say that $\C$ is \emph{PAC learnable using hypothesis class $\mathcal{H}$} and sample complexity function $p(\cdot,\cdot,\cdot,\cdot)$ if there exists an algorithm $\A$ that satisfies the following:
for all $n\in\N$, for every $c\in\C_n$, for every $D$ over $\X_n$, for every $0<\epsilon<1/2$ and $0<\delta<1/2$, if whenever $\A$ is given access to $m\geq p(n,1/\epsilon,1/\delta,\text{size}(c))$ examples drawn i.i.d. from $D$ and labeled with $c$, $\A$ outputs a polynomially evaluatable $h\in\mathcal{H}$ such that with probability at least $1-\delta$, 
\begin{equation*}
\Prob{x\sim D}{c(x)\neq h(x)}\leq \epsilon\enspace.
\end{equation*}
We say that $\C$ is statistically efficiently PAC learnable if $p$ is polynomial in $n,1/\epsilon$, $1/\delta$ and size$(c)$, and computationally efficiently PAC learnable if $\A$ runs in polynomial time in $n,1/\epsilon$, $1/\delta$ and size$(c)$.
\end{definition}

The setting where $\C=\mathcal{H}$ is called \emph{proper learning}, and \emph{improper learning} otherwise.
The PAC setting where the guarantees hold for any distribution is called \emph{distribution-free}.

\subsection{Robust Learnability}

\begin{definition}[Efficient Robust Learnability, \cite{gourdeau2021hardness}]
\label{def:robust-learning}
Fix a function $\rho:\N\rightarrow\N$. We say that an algorithm $\A$
\emph{efficiently} $\rho$-\emph{robustly learns} a concept class $\C$
with respect to distribution class $\mathcal{D}$ if there exists a
polynomial $\poly(\cdot,\cdot,\cdot,\cdot)$ such that for all
$n\in\mathbb{N}$, all target concepts $c\in \C_n$, all distributions
$D \in \mathcal{D}_n$, and all accuracy and confidence parameters
$\epsilon,\delta>0$, if $m \geq
\poly(n,1/\epsilon,1/\delta,\text{size}(c))$, whenever $\A$ is given access to
a sample $S\sim D^m$ labelled according to $c$, it outputs a polynomially evaluable function
$h:\{0,1\}^n\rightarrow\{0,1\}$ such
that $\Prob{S\sim D^m}{\risk_\rho (h,c)<\epsilon}>1-\delta$.
\end{definition}

\subsection{Local Membership Queries and Robust Learning }
\label{app:lmq}

We recall the formal definition of the LMQ model from \citep{awasthi2013learning}, but where we have changed the standard risk to the robust risk.
Here, given a sample $S$ drawn from the example oracle, a membership query for a point $x$ is $\lambda$-local if there exists $x'\in S$ such that $x\in B_\lambda(x')$.

\begin{definition}[$\lambda$-$\LMQ$ Robust Learning]
\label{def:lmq}
Let $\X$ be the instance space, $\C$ a concept class over $\X$, and $\D$ a class of distributions over $\X$. We say that $\C$ is $\rho$-robustly learnable using $\lambda$-local membership queries with respect to $\D$ if there exists a learning algorithm $\A$ such that for every $\epsilon > 0$, $\delta > 0$, for every distribution $D\in\D$ and every target concept $c\in\C$, the following hold:
\begin{enumerate}
\item $\A$ draws a sample $S$ of size $m = \poly(n, 1/\delta, 1/\epsilon,\text{size}(c))$ using the example oracle $\EX (c, D)$
\item Each query $x'$ made by $\A$ to the $\LMQ$ oracle is $\lambda$-local with respect to some example $x \in S$
\item  $\A$ outputs a hypothesis $h$ that satisfies $\risk_\rho^D(h,c)\leq \epsilon$ with probability at least $1-\delta$ 
\item The running time of $\A$ (hence also the number of oracle accesses) is polynomial in $n$, $1/\epsilon$, $1/\delta$, $\text{size}(c)$ and the output hypothesis $h$ is polynomially evaluable.
\end{enumerate}
\end{definition}

\subsection{Complexity Measures}
\label{app:complexity}

For a more in-depth introduction to these concepts, we refer the reader to \cite{mohri2012foundations}.

\begin{definition}[Shattering]
Given a class of functions $\F$ from input space $\X$ to $\set{0,1}$, we say that a set $S\subseteq\X$ is \emph{shattered by $\F$} if all the possible dichotomies of $S$ (i.e., all the possible ways of labelling the points in $S$) can be realized by some $f\in\F$. 
\end{definition}

\begin{definition}[$\VC$ Dimension]
The $\VC$ dimension of a hypothesis class $\mathcal{H}$, denoted $\VC(\mathcal{H})$, is the size $d$ of the largest finite set that can be shattered by $\mathcal{H}$.
If no such $d$ exists then $\VC(\mathcal{H})=\infty$.
\end{definition}

\begin{definition}[Littlestone Tree]
A Littlestone tree for a hypothesis class $\mathcal{H}$ on $\X$ is a complete binary tree $T$ of depth $d$ whose internal nodes are instances $x\in\X$.
Each edge is labeled with $-$ or $+$ and corresponds to the potential labels of the parent node.
Each path from the root to a leaf must be consistent with some $h\in\mathcal{H}$, i.e. if $x_1,\dots,x_d$ with labelings $y_1,\dots,y_d$ is a path in $T$, there must exist $h\in\mathcal{H}$ such that $h(x_i)=y_i$ for all $i$. 
\end{definition}

\begin{definition}[Littlestone Dimension]
The Littlestone dimension of a hypothesis class $\mathcal{H}$, denoted $\Lit(\mathcal{H})$, is the depth $d$ of the largest Littlestone tree for $\mathcal{H}$. If no such $d$ exists then $\Lit(\mathcal{H})=\infty$.
\end{definition}

\subsection{Online Learning}
\label{app:online}

In online learning, the learner is given access to examples \emph{sequentially}.
At each time step, the learner receives an example $x$, predicts its label using its hypothesis $h$, receives the true label $y$ and updates its hypothesis if $h(x)\neq y$. 
A fundamental difference between PAC learning and online learning is that, in the latter, there are no distributional assumptions.
Examples can be given adversarially, and the performance of the learner is evaluated with respect to the number of mistakes it makes compared to the ground truth.

\begin{definition}[Mistake Bound] For a given hypothesis class $\C$ and instance space $\X = \bigcup_n \X_n$, we say that an
algorithm $\A$ learns $\C$ with mistake bound $M$ if $A$ makes at most
$M$ mistakes on any sequence of examples consistent with a concept $c \in \C$. 
In the mistake bound model, we usually require that $M$ be polynomial in $n$ and size$(c)$.
\end{definition}

We now recall the Standard Optimal Algorithm (SOA) \citep{littlestone1988learning}, which has a mistake bound $M=\Lit(\C)$ when given concept class $\C$.

\begin{algorithm}
\caption{Standard Optimal Algorithm (SOA) from \cite{littlestone1988learning}}
\begin{algorithmic}
\renewcommand{\algorithmicrequire}{\textbf{Input:}}
\renewcommand{\algorithmicensure}{\textbf{Output:}}
\Require A hypothesis class $\mathcal{H}$
\State $V_1\gets \H$
\For {$t=1,2,\dots$}
\State Receive example $x_t$
\State $V^{(b)}_t \gets \set{h\in V_t\given h(x_t)=b}$
\State $\hat{y_t}=\arg\max_b \Lit(V^{(b)}_t)$ \Comment{Predict label acc. to subclass with larger Littlestone dimension}
\State Receive true label $y_t$
\State $V_{t+1} \gets V^{(y_t)}_t$
\EndFor
\end{algorithmic}
\end{algorithm}

It is well-known that the Littlestone dimension of threshold functions is infinite. The proof, which is referred to in the main text, is included below.

\begin{lemma}
The class of thresholds functions $\thresholds=\bigcup_{a\in\R}\mathbf{1}[x\geq a]$  has infinite Littlestone dimension.
\end{lemma}
\begin{proof}
Consider the interval $[0,1]$.
At each depth $i$ of the Littlestone tree, the set of nodes from left to right is $\set{\frac{j+1}{2^i}}_{j=0}^{2^{i-1}}$, and the labelling of all the left edges is $1$ and $0$ for right edges.
For a given depth $i$, a path $p$ from the root to node $x_{i,j}:=\frac{j+1}{2^i}$ for some $j\in \set{0,1,\dots,2^{i-1}}$ (including $x_{i,j}$'s label) is thus consistent with the threshold function $\mathbf{1}[x\geq x^*]$ where $x^*$ is the deepest node in $p$ (inclusive of $x_{i,j}$) that is positively labelled.
\end{proof}

\section{Useful Results}
\label{app:useful}

\subsection{Robust Risk Bounds}

\begin{lemma}[Lemma 6 in \cite{gourdeau2019hardness}]
\label{lemma:rob-triangle}
Let $c_1,c_2\in\set{0,1}^\X$ and fix a distribution $D$ on $\X$. 
Then, for all $h:\boolhc\rightarrow\set{0,1}$, 
\begin{equation*}
R_\rho^D(c_1,c_2)\leq R_\rho^D(c_1,h) + R_\rho^D(c_2,h)
\enspace.
\end{equation*}
\end{lemma}

\begin{lemma}[Lemma 14 in~\cite{gourdeau2022sample}]
\label{lemma:rob-loss-mon-conj}
Under the uniform distribution, for any $n\in\N$, disjoint $c_1,c_2\in{\MonConj}$ of even length $3\leq l\leq n/2$  on $\boolhc$ and robustness parameter $\rho= l/2$, we have that $\risk_\rho^D(c_1,c_2)$ is bounded below by a constant that can be made arbitrarily close to $\frac{1}{2}$ as $l$ (and thus $\rho$) increases. 
\end{lemma}

\begin{remark}
Note that the statement and proof of the above lemma remains unchanged if considering disjoint conjunctions, as opposed to monotone conjunctions.
\end{remark}

\subsection{Mistake Bounds for Winnow and Perceptron}

Now, we recall the mistake upper bound for Winnow in the special case of $\Halfspaces_{\boolhc}^{W+}$, where the weights are positive integers\footnote{See \href{https://www.cs.utexas.edu/~klivans/05f7.pdf}{https://www.cs.utexas.edu/~klivans/05f7.pdf} for a full derivation.} and the mistake bound for the Perceptron algorithm.

\begin{theorem}[\citep{littlestone1988learning}]
\label{thm:winnow}
The Winnow algorithm for learning the class $\Halfspaces_{\boolhc}^{W+}$  makes at most $O(W^2\log(n))$ mistakes.
\end{theorem}

\begin{theorem}[Mistake Bound for Perceptron, Margin Condition; Theorem~7.8 in \cite{mohri2012foundations}]
\label{thm:mistake-bound-perceptron}
Let $\x_1, \dots, \x_T \in \R^n$ be a sequence of $T$ points with $\norm{\x_t}\leq r $ for all $1\leq t \leq T$ for some $r>0$.
Assume that there exists $\gamma>0$ and $\vv\in\R^n$ such that for all $1\leq t \leq T$, $\gamma \leq \frac{y_t(\vv\cdot \x_t)}{\norm{\vv}}$.
Then, the number of updates made by the Perceptron algorithm when
processing $\x_1, \dots, \x_T$  is bounded by $r^2/\gamma^2$.
\end{theorem}

\subsection{Quantifier Elimination}

\begin{theorem}[Theorem 1.2 in \cite{renegar1992computational}]
\label{thm:renegar}
Let $\Psi$ be a formula in the first-order theory of the reals of the form 
$$(Q_1 x^{[1]}\in\R^{n_1})\dots (Q_\omega x^{[\omega]}\in\R^{n_\omega})P(x^{[1]},\dots,x^{[n_\omega]},y)\enspace,$$
with free variables $y=(y_1,\dots,y_l)$, quantifiers $Q_i$ ($\exists$ or $\forall$) and quantifier-free Boolean formula $P(x^{[1]},\dots,x^{[n_\omega]},y)$ with $m$ atomic predicates consisting of polynomial inequalities of degree at most $d$.  
There exists a quantifier elimination method which constructs a quantifier-free formula $\Phi$ of the form
$$\bigvee_{i=1}^I \bigwedge_{j=1}^{J_i} (h_{ij}(y)\Delta_{ij} 0)
\enspace,$$
where 
\begin{align*}
I&\leq (md)^{2^{O(\omega)}l\prod_k n_k}\\
J_i&\leq (md)^{2^{O(\omega)}\prod_k n_k}\\
\deg(h_{ij})&\leq (md)^{2^{O(\omega)}\prod_k n_k}\\
\Delta_{ij}&\in\set{\leq,\geq,=,\neq,>,<} 
\enspace.
\end{align*}

\end{theorem}

\section{Erratum on $\LEQ$ Upper Bound for Linear Classifiers in $\R^n$}
\label{app:erratum}

The statement of Theorem~11 in \cite{gourdeau2022when} (reproduced below) imposed a margin only on points in the support of the distribution (hence, conditions (i) and (ii) were on points $x\in\supp(D)$, not $x\in B_\rho(\supp(D))$).

\begin{theorem}[Theorem~11 in \cite{gourdeau2022when}]
\label{thm:s+qc-ub-ltf-real-erratum}
Fix constants $B,\gamma>0$.
Let $\mathcal{L}=\set{(c,D)\given c\in\ltfreal,D\in\D}$ be a family of halfspace and distribution pairs, where each pair $(c,D)$ with $c(x)=a^\top x +a_0$ is such that if $x\in\supp(D)$, then (i) $\norm{x}_2\leq B$ and (ii) $\gamma \leq \frac{c(x)(a^\top x)}{\norm{x}_2}$, i.e., $D$ has support bounded by $B$ and induces a margin of $\gamma$ w.r.t. $c$.
Let the adversary's budget be measured by the $\ell_2$ norm.
Then, $\mathcal{L}$  is $\rho$-robustly learnable using the $\EX$ and $\rho$-$\LEQ$ oracles with sample complexity $m=O(\frac{1}{\epsilon}( n^3 + \log (1/\delta)))$ and query complexity $r=\frac{mB^2}{\gamma^2}$.
Note that this is query-efficient if $\frac{B^2}{\gamma^2}=\poly(n)$.
\end{theorem}

The above theorem is incorrect: if the target function $c$ crosses the perturbation region $B_\rho(x)$ of a point $x$, then the infinite Littlestone dimension argument giving an infinite mistake bound applies in this case, and we cannot bound the number of counterexamples, and hence the query complexity.
This is illustrated in Example~\ref{ex:threshold-infinite-mistake}.

\begin{example}
\label{ex:threshold-infinite-mistake}
Let the input space $(\X,d)$ be $\X=[-1,1]$ with the Euclidean distance\footnote{Equivalent to the $\ell_2$ norm in this case.} $d(x,y)=\abs{x-y}$, the target concept $c(x)=\mathbf{1}[x\geq 0]$ and the adversarial budget $\rho=3/2$. 
Let the distribution $D$ on $\X$ be $D(-1)=D(1)=1/2$, thus $\supp(D)=\{-1,1\}$ and the margin is $1$.
Clearly, $B_{3/2}(-1)\cap B_{3/2}(1)=[-1/2,1/2]$, and so an infinite number of counterexamples can be given by an adversary.
\end{example}

If we require instead that $x\in B_\rho(\supp(D))$ to get conditions (i) and (ii), we end up with a sufficiently large margin that guarantees the separation of the $\rho$-expansions of the two classes, in which case it is possible to obtain trivial query complexity upper bounds:

\begin{theorem}
\label{thm:s+qc-ub-ltf-real}
Fix constants $B,\gamma>0$.
Let $\mathcal{L}=\set{(c,D)\given c\in\ltfreal,D\in\D}$ be a family of halfspace and distribution pairs, where each pair $(c,D)$ with $c(x)=\sgn(a^\top x +a_0)$ is such that if \red{$x\in B_\rho(\supp(D))$}, then (i) $\norm{x}_2\leq B$ and (ii) $\gamma \leq \frac{c(x)(a^\top x)}{\norm{x}_2}$, i.e., $D$ has support bounded by $B$ and induces a margin of $\gamma$ w.r.t. $c$.
Let the adversary's budget be measured by the $\ell_2$ norm.
Then, $\mathcal{L}$  is $\rho$-robustly learnable using the $\EX$ and $\rho$-$\LEQ$ oracles with sample complexity $m=O(\frac{1}{\epsilon}( n^3 + \log (1/\delta)))$ and query complexity $r=\frac{mB^2}{\gamma^2}$.
Note that this is query-efficient if $\frac{B^2}{\gamma^2}=\poly(n)$.
\end{theorem}

\begin{proof}[Proof of Theorem~\ref{thm:s+qc-ub-ltf-real}]
The sample complexity upper bound is a consequence of Corollary~\ref{cor:rvc-ltf} and Lemma~\ref{lemma:rob-vc}.
The query complexity upper bound follows from Lemma~\ref{lemma:rob-mistake-bound} and the mistake bound for the Perceptron algorithm, which appears in Appendix~\ref{app:useful}.
\end{proof}

\end{document}